\documentclass[12pt]{article}
\usepackage{amsmath}
\usepackage{titlesec}
\usepackage{bm}
\usepackage{rotating}
\usepackage{yhmath}
\usepackage{tabularx}
\usepackage{float}
\usepackage[linesnumbered,ruled,vlined]{algorithm2e}

\usepackage{indentfirst}
\DeclareSymbolFont{largesymbol}{OMX}{yhex}{m}{n}
\DeclareMathAccent{\Widehat}{\mathord}{largesymbol}{"62}

\usepackage{graphicx}
\usepackage{url} 
\usepackage{makecell}
\allowdisplaybreaks[4]
\RequirePackage[colorlinks,citecolor=blue,urlcolor=blue]{hyperref}
\usepackage[round,authoryear]{natbib}
\usepackage{authblk}
\usepackage{amsthm}
\usepackage{amsfonts}
\usepackage{cleveref}
\usepackage{graphicx}
\graphicspath{{fig}}
\usepackage{epstopdf}
\usepackage{amsopn}
\usepackage{mathrsfs}
\usepackage{enumitem}
\usepackage{makecell}
\usepackage{colortbl}
\usepackage{booktabs}
\usepackage{xcolor}
\usepackage{soul}
\usepackage{longtable}
\usepackage{multirow}
\usepackage{tikz}
\usepackage{dsfont}
\usepackage{makecell}
\usepackage{amssymb}
\usepackage{bbm}
\usepackage{enumitem}   
\usepackage{xr}
\usetikzlibrary{decorations.pathreplacing}

\ifpdf
  \DeclareGraphicsExtensions{.eps,.pdf,.png,.jpg}
\else
  \DeclareGraphicsExtensions{.eps}
\fi

\graphicspath{{pics/}}

\newtheorem{assumption}{Assumption}
\newtheorem{lemma}{Lemma}
\newtheorem{theorem}{Theorem}

\newtheorem{remark}{Remark}

\DeclareMathOperator*{\argmin}{arg\,min}

\makeatletter
\newcommand*{\addFileDependency}[1]{
  \typeout{(#1)}
  \@addtofilelist{#1}
  \IfFileExists{#1}{}{\typeout{No file #1.}}
}
\makeatother

\usepackage{caption}

\newcommand{\hstar}{\h^{*}}

\newcommand{\Qstar}{Q^{*}}

\newcommand{\hcL}{\hat{\cL}}
\newcommand{\XT}{X_{T}}
\newcommand{\YT}{Y_{T}}

\newcommand{\hR}{\hat{R}}

\newcommand{\cQ}{\mathcal{Q}}
\newcommand{\q}{\mathbf{q}}

\newcommand{\mR}{\mathbb{R}}

\newcommand{\hQ}{\hat{Q}}

\newcommand{\hV}{\widehat{\mathcal{V}}}

\newcommand{\vect}[1]{\ensuremath{\mathbf{#1}}}

\newcommand{\mat}[1]{\ensuremath{\mathbf{#1}}}

\newcommand{\mE}{\mathbb{E}}
\renewcommand{\Pr}{\mathbb{P}}

\newcommand{\tlO}{\tilde{O}}

\newcommand{\A}{\mat{A}}

\newcommand{\X}{\mat{X}}

\renewcommand{\H}{\mat{H}}

\newcommand{\F}{\mat{F}}

\newcommand{\h}{\vect{h}}

\newcommand{\x}{\vect{x}}

\newcommand{\y}{\vect{y}}
\newcommand{\z}{\vect{z}}
\newcommand{\s}{\vect{s}}

\newcommand{\f}{\vect{f}}

\newcommand{\vt}{\vect{t}}
\newcommand{\vs}{\vect{s}}
\newcommand{\vz}{\vect{z}}
\newcommand{\vy}{\vect{y}}

\newcommand{\cF}{\mathcal{F}}
\newcommand{\cG}{\mathcal{G}}

\newcommand{\cN}{\mathcal{N}}
\newcommand{\cD}{\mathcal{D}}
\newcommand{\cL}{\mathcal{L}}
\newcommand{\cH}{\mathcal{H}}
\newcommand{\cT}{\mathcal{T}}

\newcommand{\cW}{\mathcal{W}}
\newcommand{\cS}{\mathcal{S}}
\newcommand{\cX}{\mathcal{X}}
\newcommand{\cY}{\mathcal{Y}}

\newcommand{\cV}{\mathcal{V}}
\newcommand{\cU}{\mathcal{U}}

\newcommand{\Perp}{\perp\!\!\!\perp}

\usepackage{algorithm}
\usepackage{algorithmic}

\newcommand{\blind}{1}

\begin{document}

\def\spacingset#1{\renewcommand{\baselinestretch}%
{#1}\small\normalsize} \spacingset{1}

\date{}


\if1\blind
{
  \title{\bf
Deep Transfer Learning: Model Framework and Error Analysis
}
\author[1,2]{\small Yuling Jiao}
\author[3]{Huazhen Lin}
\author[1]{\small Yuchen Luo}
\author[1,2]{\small Jerry Zhijian Yang}
\affil[1]{\it \small School of Mathematics and Statistics, Wuhan University, Wuhan, Hubei 430072, China}
\affil[2]{\it \small Hubei Key Laboratory of Computational Science, Wuhan, Hubei 430072, China}
\affil[3]{\it \small Center of Statistical Research and School of Statistics, Southwestern University of Finance and Economics, Chengdu 61113, China}
\setcounter{Maxaffil}{0}
		
		\renewcommand\Affilfont{\itshape\small}
    \maketitle
} \fi

\if0\blind
{
  \bigskip
  \bigskip
  \bigskip
  \begin{center}
    {\LARGE\bf Deep Transfer Learning: Model Framework and Error Analysis}
\end{center}
  \medskip
} \fi

\begin{abstract}
This paper introduces a framework for deep transfer learning  aimed at improving performance on a single-domain downstream task with a limited sample size ($m$) by leveraging information from multi-domain upstream data with a significantly larger sample size ($n$), where $m \ll n$.
Our framework offers several intriguing features. First,
it allows the existence of both shared and domain-specific features across multi-domain  data and provides a framework for  automatic identification,  achieving precise transfer and utilization of information.
Second, the framework explicitly identifies upstream features that contribute to downstream tasks, establishing clear relationships between upstream domains and downstream tasks, thereby enhancing interpretability.
Error analysis shows that our framework can significantly improve the convergence rate for learning Lipschitz functions in downstream supervised tasks, reducing it from $\tlO(m^{-\frac{1}{2(d+2)}} + n^{-\frac{1}{2(d+2)}})$ (for ``no transfer”) to $\tlO(m^{-\frac{1}{2(d^*+3)}} + n^{-\frac{1}{2(d+2)}})$ (for ``partial transfer”), and even to $\tlO(m^{-1/2} + n^{-\frac{1}{2(d+2)}})$ (for ``complete transfer”), where $d^* \ll d$ and $d$ is the dimension of the observed data. Our theoretical findings are supported by empirical experiments on image classification and regression datasets.
\end{abstract}

\noindent%

\spacingset{1.9} 

\section{Introduction}
\label{sec:intro}

Transfer learning (TL) is a powerful technique that improves model performance in target domains with limited data by leveraging knowledge from related source domains. This approach addresses the common challenge of obtaining large datasets necessary for high predictive accuracy. 
TL has been successfully applied in diverse fields, including natural language processing \citep{ruder2019transfer}, medical image analysis \citep{kora2022transfer}, drug discovery \citep{cai2020transfer}, and materials science \citep{jha2019enhancing, kim2021deep, ju2021exploring}. 

Due to 
its notable successes in practical applications, 
 extensive research has focused on theory of transfer learning under conditions of covariate shift and posterior shift, addressing both parametric and nonparametric models. 
Specifically, covariate shift refers to changes in the marginal distributions of covariates between the source and target domains, while posterior shift involves changes in the relationships between covariates and outcomes across these domains.
Theoretical investigations into transfer learning within parametric models, such as high-dimensional linear regression \citep{li2022transfer}, generalized linear models \citep{bastani2021predicting,li2023estimation,tian2023transfer} or gaussian graphical models with false discovery rate control \citep{li2023transfer}, have been thoroughly explored. 
However, once  parametric models are specified, they often lack the flexibility to handle the complex posterior shifts encountered in real-world settings, leading to insufficient information transfer due to  the varying expressions in parameters  for the different  models. 



Nonparametric transfer learning  with covariate shift has been addressed by \cite{huang2006correcting, wen2014robust, wang2016nonparametric, cai2024transfer,schmidt2024local} for regression tasks and by 
\cite{ben2006analysis, blitzer2007learning, mansour2009domain} for  classification tasks. Theoretical results 
have been developed for both the covariate shift setting \citep{shimodaira2000improving, sugiyama2007direct} and the posterior drift setting \citep{cai2021transfer, reeve2021adaptive}. Paritcularly,  minimax rates are derived under posterior drift by \cite{cai2024transfer} and under covariate shift by \cite{kpotufe2021marginal}, respectively. 
These nonparametric methods  are primarily based on traditional nonparametric approximation techniques, such as local polynomial methods \citep{fan2018local}, spline approximation \citep{schumaker2007spline}, and reproducing kernel techniques \citep{berlinet2011reproducing, lv2018oracle}. In practice, these methods often face the ``curse of dimensionality" when the covariate dimension $d$  is moderate to high, for instance, when 
$d>3$.

Thanks to their powerful function-fitting capability, well-designed architectures, effective training algorithms, and high-performance computing technologies, nonparametric analysis based on deep neural networks (DNNs) has achieved tremendous success.  Deep transfer learning, which uses DNNs to transfer learned representations across tasks, has been studied in the literature. For example, \cite{du2020few}, \cite{tripuraneni2020theory}, \cite{tripuraneni2021provable}, \cite{chen2021weighted}, and the references therein develop theories that consistently demonstrate the improvement of deep transfer learning in terms of sampling efficiency, under the assumption that sufficient representations of the downstream task can be obtained from the upstream domains and  certain transferability conditions. However, these studies do not account for cases where downstream tasks may require more than representations which the upstream tasks can provide. 

While prior research has advanced our understanding of transfer learning, the existing literature lacks systematic research on the following issues, particularly within the framework of deep learning:
\begin{itemize}
\item[(i)]
 How to learn a sufficient and invariant representation form the upstream data?
\item[(ii)] When the upstream representation is sufficient, which features are relevant for downstream prediction?  And how should transfer learning be described  when the upstream representation is not sufficient for the downstream task?
\item[(iii)] Given that the representations of different datasets should not be identical and there should be shared as well as domain-specific features, how can we design a model framework to accurately  estimate and characterize  them?
\end{itemize}
In this paper, we aim to address the aforementioned questions by proposing a deep  learning  framework for transfer learning. 
Our model accomplishes this in three key aspects. 
Firstly, our model  permits  the presence of  common and distinct features across domains in the upstream and downstream data and provides framework to automatically identifies them. Secondly, it explicitly expresses which upstream features contribute to  the downstream tasks,  establishing  connection among  domains  and enhancing interpretability.  In addition, a non-linear 
function $Q^*$ is introduced for the downstream task to learn features that are independent of the upstream representation, determining the level of difficulty in transfer learning. 
Particularly, in the case where  
 $Q^*=0$, the upstream representation is sufficient to the downstream task,  transfer from upstream to downstream is facile, leading to ``complete   transfer" and a parametric convergence rate.  
On the other hand,   when specific information is encoded in the non-zero $Q^*$, the transfer problem becomes  more challenging.  Our error analysis reveals that the convergence speed 
varies 
from slow $\mathcal{O}(m^{-\frac{\beta}{2(d+1+2\beta)}})$, 
to moderate {$\mathcal{O}(m^{-\frac{\beta}{2(d^*+1+2\beta)}})$} with $d^*\ll d$. 
These rates correspond to ``no  transfer"  that the downstream task necessitates learning a $d$-dimensional function, and ``partial transfer" problems where the downstream task entails learning an $d^*$-dimensional function. 
Importantly, we emphasize that our model framework seamlessly adapts to these  scenarios, accommodating varying degrees of information transfer and yielding distinct convergence rates for the downstream task.

The remainder of this paper is structured as follows: Preliminary details about the framework, including essential definitions and model setup, are provided in Section \ref{sec:prelim}. Our main theoretical results are presented in Section \ref{sec: theoretic result for deep transfer learning}. Experimental results are showcased in Section \ref{sec: Experiment}, followed by a discussion in Section \ref{sec: Conclusion}. Technical proofs are provided in the Supplementary Materials. 

\section{Preliminaries}

\label{sec:prelim}
\textbf{Notation: } We use bold letters (e.g., $\x$, $\F$) to refer to vectors or matrixs. The norm $\Vert \cdot \Vert$ appearing on a vector or matrix refers to its $\ell_2$ norm or \textit{infinity-norm} respectively. The norm $\Vert \cdot \Vert_{1}$ appearing on a vector or matrix refers to its $\ell_1$ norm or group $\ell_1$ norm respectively. We use the bracketed notation $[n] = \{1, \hdots, n\}$ as shorthand for integer sets. $\Vert f \Vert_{L}$ is the Lipschitz constant of function $f$. We use $\cU_{r}$ as $U[0,1]^{r}$, which is the random vector with each coordinate being mutually independent and have uniform distribution $U[0,1]$.  Throughout we will use
$\cH$ to be a function class of features mapping $\mR^d \to \mR^r$, $\cG$ to be a function class of lipschitz-constrained mapping $ \mR^r \to \mR$ and $\mathcal{Q}$ to be a function class during fine-tuning.
We use $\tlO$ to denote an expression that hides polylogarithmic factors in all problem parameters. 

\subsection{Transfer learning based on sufficient and domain-invariant representation}
\subsubsection{Setup}\label{setup}
{Formally, the upstream dataset $\mathcal{D}=\{ (\X_i, Y_i, S_{i})\}_{i=1}^{n}$ consists of $n$ independent and identically distributed (\textit{i.i.d}.) copies of random triples $(\X, Y, S)$, sampled from an unknown distribution $\Pr(\x, y, \s)$  supported  over $\cX \times \cY \times \cS$, 
where $\cX \subset \mathbb{R}^d$, $\cY \subset \mathbb{R}^1$ and $\cS = [p]$ denotes the domain label set.  
We combine all data together to define ($\X, Y$) as their joint distribution is $\int \Pr(\x, y, \s) d \s$. 
For a given domain $s\in \cS$, suppose that there  exists a low-dimensional features  $\hstar_s(\X)$ that capture the relationship between  $\X$ and $Y$.  By  allowing  $\hstar_s(\X) \cap \hstar_t(\X) \neq \varnothing, s, t \in \mathcal{S}$, we enable the features to vary or be shared across domains. Specifically, domain-specific features reduce transfer bias, enhancing accuracy, while shared features, estimated based on multiple domains, improve efficiency. Here, we impose no restrictions on the patterns of variation or sharing across domains. 
Denote $\hstar(\X)\in \mathbb{R}^{r \times 1}, r\ll d$, as the combination of $\hstar_s(\X), s \in [p] $, noting that $\X_s \Perp Y_s \mid \hstar_s(\X)$ for any given domain $s$, we then have
\begin{align}
   \X \Perp Y \mid \hstar(\X), \label{suff}\\
    S \Perp \hstar(\X),\label{inv}
\end{align}
where $\X \Perp Y \mid \hstar(\X)$ indicates that the information of $\X$ to predict $Y$ is \textbf{sufficiently} encoded in $\hstar(\X)$, and $S \Perp \hstar(\X)$ signifies the need for the representation to be \textbf{invariant} across all domains in the upstream dataset.
Under \eqref{suff}, and employing the shorthand notation $\X_{S}, Y_{S}$ to represent the triplet $(\X, Y, S)$,
we use the following regression model  for identifying and estimating $\hstar(\X)$:
\begin{equation}
Y_S = \F^{*}_{S} \hstar(\X_{S}) + \varepsilon_{1}, \quad \mE[\varepsilon_1]=0,
\label{eq:sourcedata}
\end{equation}
where the sparse vector $\F^{*}_{S} \in \mR^{1 \times r}$ represents the transformation applied to the representation $\hstar(\X_S)$ for domain $S$, 
$\varepsilon_{1}$ is a bounded noise with  mean zero.  The sparsity of $\F^{*}_{S}$ is used to identify the feature subset  $\hstar_S(\X_{S})(\subset\hstar(\X_S))$ that 
capture the relationship between $\X_S$ and $Y_S$.
The sparsity sets our framework apart from existing transfer learning methods by allowing features to vary or be shared across domains, rather than exclusively  being shared. 
With the sparsity of $\F_S$, model (\ref{eq:sourcedata}) also provides a more parsimonious and interpretable representation for 
each domain. Domains with different feature 
sets enable us to better explore their interdomain  differences.
For the downstream data, we focus on single-target domain transfer learning. Let the data set $\mathcal{T}=\{ (\X_{T, i}, Y_{T, i})\}_{i=1}^{m}$ consist of $m$ samples of the random variable pair $(X_{T}, Y_{T})$, which are i.i.d. from an unknown distribution $\Pr_{0}$, which is also supported over $\cX \times \cY$.
We assume the supervised downstream task data following the regression model: 
\begin{equation}\label{eq:targetdatareg}
    \YT = \F_{T}^{*}\hstar(\XT) + \Qstar (\XT)+ \varepsilon_{2}, \quad \mE[\varepsilon_2]=0,
\end{equation}
or the binary classification model
    \begin{equation}\label{eq:targetdatacl}\YT \sim \mathrm{Ber}(\sigma(\F_{T}^{*}\hstar(\XT) + \Qstar (\XT)))
\end{equation}
where $\varepsilon_{2}$ is bounded noise with zero mean, $\F_{T}^{*}$ is the sparse transformation vector for target domain and $\Qstar: \mR^{d} \to \mR$ is an additional function to be learned which indicates that downstream tasks may not share sufficient similarity to predict $\YT$ solely based on the representation $\hstar({\XT})$. Thus, $\Qstar \neq 0 $ serves the purpose of extracting additional information different from $\hstar(\XT)$ 
which is implemented  by forcing  
\begin{equation}\label{qind}
\Qstar (\XT)\Perp \hstar(\XT).
\end{equation}
To explicitly express the constraint (\ref{qind}), we  further model 
\begin{equation}\label{modelQ}
\Qstar(\XT)=q^{*}( \A^*\XT),
\end{equation}
where $\A^* \in \mathbb{R}^{d} \rightarrow \mathbb{R}^{d^{*}}$ with $d^* \leq  d$, and $q^* :\mathbb{R}^{d^*} \mapsto \mathbb{R}$. 
In fact, if $d^* = d$, setting $\A^*=\mathbf{I}$ reduces $q^{*}(\A^*\XT)$ to $q^{*}(\XT)$ or $\Qstar(\XT)$,
indicating no restriction on the model (\ref{modelQ}). Therefore, the value of $d^*$ indicates the strength of the constraint (\ref{qind}). A smaller $d^*$ implies a more stringent constraint of (\ref{qind}). In practice, we can determine $d^*$ among various candidate integers by evaluating prediction accuracy.

When $\F_T^{*} = 0$, 
the contribution of the upstream representation to the downstream task is negligible, and we refer to this scenario as “no transfer”. On the other hand, when $\Qstar = 0$, the entirety of upstream information can be seamlessly transferred to the downstream task. In this case of “complete transfer”, the task involves solely the estimation of the sparse vector $\F_{T}^{*}$ through penalized least squares.
In situations where $\F^{*}_T \neq 0$ and $Q^*\neq 0$,  
we classify it as ``partial transfer". Hence, our downstream task model (\ref{eq:targetdatareg})-(\ref{modelQ}) can adaptively accommodate these three distinct scenarios.
Furthermore, {our Theorem \ref{fine-tuning result}, Remark \ref{remark about comparison of theory results} and Theorem \ref{theorem: special case of sparse erm during fine-tuning} } demonstrate that the end-to-end convergence rate distinguishes these  settings.
 \begin{remark}
Practitioners have also explored the following transfer models, 
\begin{equation}\label{eq:targetdataregf}
    \YT = f^{*}_{T}(\hstar(\XT)) + \varepsilon_{2}, \quad \mE[\varepsilon_2]=0,
\end{equation}
or the binary classification model
\begin{equation}\label{eq:targetdataclf}\YT \sim \mathrm{Ber}(\sigma( f^{*}_{T}(\hstar(\XT)) ))
\end{equation}
where $f^*_{T}:\mathbb{R}^r \to \mathbb{R}$ is a low dimensional function to be learned nonparametrically in the downstream task. 
{The model (\ref{eq:targetdataregf}) or (\ref{eq:targetdataclf}) suggests sufficient features from the upstream domains to the downstream task, as $ f^{*}_{T}(\hstar(\XT))$  depends on $\hstar(\XT)$. Here in (\ref{eq:targetdatareg}) or (\ref{eq:targetdatacl}) , we choose a linear sparse function over the nonparametric function $ f^{*}_{T}(\cdot)$ for interpretability, as well as the consideration  that the inherent nonparametric features are already captured by $\hstar(\X_T)$.
}
 \end{remark}

\subsubsection{Loss for representation learning}\label{secup}
Notably, for any bijective mapping $\psi$ acting on $\hstar$, the composition $\psi \circ \hstar$ preserves the sufficiency and invariance properties required in (\ref{suff})-(\ref{inv}). As a result, $\hstar(\X)$ is not unique. To ensure uniqueness, following the standard approach in dimensionality reduction \citep{chen2024deep,liu2024encoding}, we constrain the distribution of $\hstar(\X)$ to match a given distribution $\cU_{r}$. Specifically, we assume a uniform distribution for $\cU_{r}$, as its existence can be guaranteed by optimal transport theory \citep{villani2009optimal}. Since the entries of $\cU_{r}$ are independent, this promotes the disentanglement of $\hstar(\X)$.

As discussed in Section \ref{setup}, for different domain from upstream,  
only a portion of information from $\hstar(X_{S})$ is utilized during the prediction process. Consequently, we incorporate domain-specific extractors represented by the sparse vector $\F_{S}$, 
and add $\ell_1$ penalty in loss for the sparsity of $\F_{S}$. 
For simplicity, we denote $(\F_{1}, \hdots, \F_{p})$ as $\F \in \mathbb{R}^{pr}$.
Then, defining the  population risk as:
\begin{equation}
\label{eq:poploss}
\begin{aligned}
    R(\F,\h) = &\mE[\|Y_S-\F_S\h(\X_S)\|^2] + \lambda \mathcal{V}[\h(\X), S] + \tau \cW_1(\h(\X), \cU_{r}) + \mu   \| \F\|_{1},
\end{aligned}
\end{equation}
where the first term is the regression loss for 
(\ref{eq:sourcedata}), and
$\mathcal{V}[\h(\X), S]$ is a measure used to evaluate independence between $\h(\X)$ and $S$ as required in (\ref{inv}), such as  the   distance  covariance (detail definition see Section \ref{dc and nn}),   and $$\cW_1(\h(\X), \cU_{r}) = \sup _{\|g\|_L \leq 1} \mathbb{E}_{\x \sim \h(\X)}[g(\x)]-\mathbb{E}_{\y \sim \cU_{r}}[g(\y)],$$
is the Wasserstein distance  used to match  the distribution of  $\h(\X)$ and $\cU_{r}$,  and  $\lambda$, $\tau, \mu$ are  regularization parameters. 
Given data  $\mathcal{D}= \{ (\X_i,Y_i,S_{i})\}_{i=1}^{n}$, 
the empirical version of (\ref{eq:poploss}) can be written as:
\begin{equation}
\label{eq:trainerm}
\begin{aligned}
    \hR(\F, & \h) = \sum_{i=1}^{p} \sum_{j=1}^{n_p}\frac{1}{n} (Y_{i,j} - \F_{i}\h(X_{i,j}))^2 + \lambda \hV_{n}[\h(\X), S] 
     + \tau \hat{\cW_1}(\h(\X), \cU_{r}) + \mu   \| \F\|_{1}, 
\end{aligned}
\end{equation}
where $\sum_{j = 1}^{p} n_{p}= n$, and  $\hV_{n}[\h(\X),S]$ is the empirical form of $\mathcal{V}[\h(\X), S]$ (see Section \ref{dc and nn} for detail), and $$ 
\begin{aligned}
    & \hat{\cW_1}(\h(\X), \cU_{r}) = \frac{1}{K_3}\sup _{g \in \cG }\frac{1}{n} \sum_{i=1}^{n}(g(\h(\X_{i}))-g(\vect{\xi}_i)), 
\end{aligned}
$$
denotes the empirical form of the Wasserstein distance, wherein $\{\xi_{i}\}_{i=1}^{n}$ denotes \textit{i.i.d} random samples drawn from $\cU_{r}$ and the function class $\cG$ is chosen as the  norm-constrained neural network class whose Lipschitz constants are bounded by 
 $K_3$, see Section \ref{dc and nn} for detail. 


We define the estimated sparse
vector $\hat{\F} = (\hat{\F}_{1}, \hdots , \hat{\F}_{p})$  and estimated  feature representations $\hat{\h}$ as the following 
\begin{equation}
\label{eq:erm}   
(\hat{\F},\hat{\h})= \argmin_{(\F,\h) \in \mathcal{B}_{pr}(B) \circ \cH} \hR(\F, \h),
\end{equation}
where  $\mathcal{B}_{pr}(B)$ is the ball in $\mathbb{R}^{pr}$ with radius $B$, and   $\cH$ is a deep neural network class
with prescribed network structure. 
We summarize the procedure for representation learning from upstream data in Algorithm \ref{alg: Transfer Algorithm up}.
\begin{algorithm}[tb]
   \caption{Upstream training}
   \label{alg: Transfer Algorithm up}
\begin{algorithmic}
   \STATE {\bfseries Input:} Upstream data $\{\X_{i}, Y_{i}, S_{i}\}_{i=1}^{n}$
   \STATE Initialize $\F$ 
   and $\h $ 
   \REPEAT 
   \STATE $g \leftarrow \mathop{\rm argmax}_{g \in \mathcal{G}}\hat{\mathcal{W}_{1}}(\h(X), \mathcal{U}_{r})$
   \STATE $(\F, \h) \leftarrow \mathop{\rm argmin}_{(\F, \h) \in \mathcal{B}_{q}(B) \circ \mathcal{H}} \hat{R}(\F, \h)$
   \UNTIL  convergence 
   \STATE {\bfseries Output: $\hat{\F}, \hat{\h}$}
\end{algorithmic}
\end{algorithm}


\subsubsection{Loss for downstream task}\label{secdown}
For the downstream learning, we plug in the learned $\hat{\h}$ in Algorithm \ref{alg: Transfer Algorithm up} and fit the regression and classification  models (\ref{eq:targetdatareg}) 
(\ref{eq:targetdatacl}) 
with downstream data $\cT = \{ (\X_{T,i}, Y_{T,i})\}_{i=1}^{m}$, respectively.
Specifically,  we estimate  the optimal $\F_{T}^{*}$ and $\Qstar(\cdot) = q^*(\A^{*}\cdot)$ via 
\begin{equation}
\label{finetune erm}
\begin{aligned}
    &( \hat{\F}_{T}, \hQ)   \in  \argmin_{(\F_T,\A, q) \in \mathcal{B}_{r}(B) \circ  
    \mathcal{B}_{dd^*}(B) \circ \mathcal{Q}} \hat{\cL}(\hat{\h}, Q_{q, \A}, \F_T) \\
      =&\sum_{i=1}^{m} \ell(\F_T \hat{\h}(\X_{T,i}) + q(\A\X_{T,i}), Y_{T,i})/m+ \kappa\hat{\mathcal{V}}_{m}[\hat{\h},Q_{q, \A}]+\chi \| \F_T \|_{1} + { \zeta  \|\A\|_{F}^2},
\end{aligned}
\end{equation}
where 
$Q_{q, \A}(\cdot) = q(\A\cdot)$ 
and $\ell(\cdot,\cdot)$ is the loss for regression or classification, 
$\mathcal{B}_{r}(B)$ is the ball in $\mathbb{R}^{r}$ with radius $B$, and  
$\mathcal{Q}$ is a prescribed neural network class,
and the empirical distance covariance term 
$\hat{\mathcal{V}}_{m}[\hat{\h}, Q_{q, \A}]$ enforces 
the independent constraint in (\ref{qind}),   
and $\kappa, \chi,\zeta$ are the regularization parameters.
Also, define the population risk in downstream as $\cL(\h, Q_{q, \A}, \F_{T}) = \mE_{\XT, Y_{T}}[\hat{\cL} (\h(\XT), q(\A \XT), \F_{T})]$.



\subsection{Function class and measure}
\label{dc and nn}
\textbf{Distance covariance.} 
We recall the concept of distance covariance \citep{szekely2007measuring}, which characterizes the dependence of two random variables.
Let $\mathfrak{i}$ be the imaginary unit $(-1)^{1 / 2}$. For any $\vt\in \mathbb{R}^{d}$ and $\vs\in \mathbb{R}^m$, let  $\psi_{ Z}(\vt)=\mathbb{E} [\exp^{\mathfrak{i}\vt^TZ}], \psi_{Y}(\vs)=\mathbb{E} [\exp^{\mathfrak{i}\vs^TY}],$ and $\psi_{Z, Y}(\vt,\vs) = \mathbb{E} [\exp^{\mathfrak{i}(\vt^TZ+ \vs^TY)}]$ be the characteristic functions of random vectors $Z\in \mR^d , Y \in \mR^m,$ and the pair  $(Z, Y)$, respectively.
The squared distance covariance $\mathcal{V}[Z, Y]$ is defined as
$$
\mathcal{V}[Z, Y]=\int_{\mathbb{R}^{d+m}} \frac{\left|\psi_{Z, Y}(\vt, \vs)-\psi_{ Z}(\vt) \psi_{Y}(\vs)\right|^2}{c_{d}c_{m}  \|\vt\|^{d+1}\|\vs\|^{m+1}} \mathrm{d} \vt \mathrm{d} \vs,
$$
where
$
c_{d}=\frac{\pi^{(d+1) / 2}}{\Gamma((d+1) / 2)}.
$
Given $n$ \textit{i.i.d} copies $\{Z_{i}, Y_{i}\}_{i=1}^n$ of $(Z,Y)$,
an unbiased estimator of $\mathcal{V}$
is the empirical distance covariance $\widehat{\mathcal{V}}_{n}$,
which can be elegantly expressed as a $U$-statistic \citep{huo2016fast}
\begin{equation}\label{usta}
\widehat{\mathcal{V}}_{n}[Z, Y]
 = \frac{1}{C_{n}^{4}} \sum_{1 \leq i_{1}<i_{2}<i_{3}<i_{4} \leq n} k\left(\left(Z_{i_{1}}, Y_{i_{1}}\right), \cdots,\left(Z_{i_{4}}, Y_{i_{4}}\right)\right), \nonumber
\end{equation}
where $k$ is the kernel defined by
\begin{eqnarray*}\label{kernel}
k\left(\left(\vz_{1}, \vy_{1}\right),\ldots,
\left(\vz_{4}, \vy_{4}\right)\right)
&=&\frac{1}{4} \sum_{\substack{1 \leq i, j \leq 4 \\ i \neq j}}\|\vz_{i}-\vz_{j}\| \|\vy_{i}-\vy_{j}\| \\
&+&\frac{1}{24} \sum_{\substack{1 \leq i, j \leq 4 \\ i \neq j}}\left\|\vz_{i}-\vz_{j}\right\| \sum_{\substack{1 \leq i, j \leq 4 \\ i \neq j}}\|\vy_{i}-\vy_{j}\| \\
& -&\frac{1}{4} \sum_{i=1}^{4}(\sum_{\substack{1 \leq j \leq 4 \\ j \neq i} }\left\|\vz_{i}-\vz_{j}\right\| \sum_{\substack{1 \leq j \leq 4 \\ i \neq j} }\|\vy_{i}-\vy_{j}\|).
\end{eqnarray*}




\textbf{Hölder class. }
Let $\beta=s+r>0, r \in(0,1]$ and $s=\lfloor\beta\rfloor \in \mathbb{N}_0$, where $\lfloor\beta\rfloor$ denotes the largest integer strictly smaller than $\beta$ and $\mathbb{N}_0$ denotes the set of non-negative integers. For a finite constant $B>0$, the Hölder class of functions $\mathcal{H}^\beta\left([0,1]^d, B\right)$ is defined as
\[
\mathcal{H}^\beta([0,1]^d, B) 
= \{f:[0,1]^d \rightarrow \mathbb{R}, \max _{\|\alpha\|_1 \leq s}\left\|\partial^\alpha f\right\|_{\infty} \leq B,  \max _{\|\alpha\|_1=s} \sup _{x \neq y} \frac{\left|\partial^\alpha f(\X)-\partial^\alpha f(y)\right|}{\|x-y\|^r} \leq B\}
\]
where $\partial^\alpha=\partial^{\alpha_1} \cdots \partial^{\alpha_d}$ with $\alpha=\left(\alpha_1, \ldots, \alpha_d\right)^{\top} \in \mathbb{N}_0^d$ and $\|\alpha\|_1=\sum_{i=1}^d \alpha_i$. 
The Lipschitz class  
$\cH^{1}:= \cH([0,1]^d, K_2)$ is a special case of  Hölder class with $\beta =1$.
We adopt the H\"{o}lder class as the evaluation class, which encompasses the true functions. This is a common setup in nonparametric estimation \citep{gyorfi2002distribution}. 

\textbf{Norm-constrained neural network.} 
A neural network function $\phi:\mathbb{R}^{N_0} \to \mathbb{R}^{N_{L+1}}$ is a function that can be parameterized  as, 
\begin{equation}
    \label{NN form}
    \phi(\X) = T_L(\sigma(T_{L-1}(\cdots \sigma(T_0(\X))\cdots))),
\end{equation}
where the activation function $\sigma(x):= x\lor 0$ is applied component-wisely and $T_l(\X) := A_l x +b_l$ is an affine transformation with $A_l \in \mathbb{R}^{N_{l+1}\times N_{l}}$ and $b_l\in \mathbb{R}^{N_{l+1}}$ for $l=0,\dots,L$. The numbers $W=\max\{N_1,\dots,N_L\}$ and $L$ are called the \textit{width} and the \textit{depth} of neural network, respectively. 
We denote by $\cN\cN_{d_1, d_2}(W, L)$ the set of ReLU FNN functions 
with width at most $W$, depth at most $L$ and its input (output) dimensions of $d_1$($d_2$). When
the input dimension and output dimension are clear from contexts, we simply denote it by $\cN\cN(W, L)$.
We define the norm constrained neural network $\mathcal{N N}(W, L, K)$ as the set of functions $\phi_\theta \in \mathcal{N N}(W, L)$ of the form \eqref{NN form} that satisfies the following norm constraint on the weights
\[
\kappa(\theta):=\left\|A_L\right\| \prod_{\ell=0}^{L-1} \max \left\{\left\|\left(A_{\ell}, \boldsymbol{b}_{\ell}\right)\right\|, 1\right\} \leq K .
\]

 It's easy to see that any $f \in \cN\cN(W, L, K)$, $\| f \|_L <K$.
We set the network class used in Section ~\ref{secup}-~\ref{secdown} as follows:  
$\cH=\cN\cN_{d, r}(W_1, L_1, K_1)$,  $\cG= \cN\cN_{r, 1}(W_2, L_2, K_2)$ and
$\mathcal{Q}=\cN\cN_{d^{*}, 1}(W_3, L_3, K)$.
\section{Main results}
\label{sec: theoretic result for deep transfer learning}
\subsection{Result for upstream training}
\label{sec: theoretic result for upstream}
We first make the following  regularity assumptions 
on the target function and on the neural network function classes.
\begin{assumption}
\label{assump: regularity} 
(Regularity conditions on excess risk.)
\begin{itemize}
\item[(A1)] 
 $\cX=[0, 1]^d$, 
and
$|Y|\leq B$  for all $Y \in \cY$;
\item[(A2)] The sparse vector 
$\|\F^*_{s}\|\leq B, \forall s\in \mathcal{S}$, and each component of $\hstar$
is contained in $\cH^{\beta}([0, 1]^d, B)$; 
\item[(A3)]  
$\sup_{\x\in \mathcal{X}}\{\|g(\h(\x))\| \lor \|\h(\x)\|\} \leq B$
for all $ g \in \cG, \h \in \cH$.
\end{itemize}
\end{assumption}
\begin{remark}
The regularity conditions in  Assumption \ref{assump: regularity} are mild. Firstly, we impose the boundedness assumptions on $\cX$ and $\cY$ in (A1) for the sake of simplicity in presentation. This choice is facilitated by the ability to employ truncation techniques, given additional exponential decay tail assumptions on the distribution of the covariate and the noise, as elaborated in detail in \cite{gyorfi2002distribution}. 
Secondly, the assumption of a bounded target (A2) is  standard  in both parametric and nonparametric estimation. We  required  the network's  output to be  bounded in (A3), which is  achievable through a  truncation implemented by a ReLU layer. 
\end{remark}



\begin{theorem}[Upstream excess risk bound]
\label{main resultup}
Suppose the Assumption \ref{assump: regularity}
 holds. 
 Set $\cH = \cN\cN_{d,r}(W_1, L_1, K_1)$ and network $\cG=  \cN\cN_{r,1}(W_2, L_2, K_2)$ according the  following Table:


\begin{table}[ht]
  \centering
  \label{tab:width-depth-norm}
  \footnotesize
  \begin{tabular}{lll}
    \toprule
    NN class  &\multicolumn{1}{c}{$\mathcal{H}$} & \multicolumn{1}{c}{$\mathcal{G}$} \\
    \midrule
    width(W) &
    $W_1 \asymp n^{\frac{2d+\beta}{4(d+1+\beta)}} \mathbb{I}_{\beta \leq 2} + n^{\frac{2d+\beta}{2(2d + 3\beta)}} \mathbb{I}_{\beta > 2}$ & $W_2 \gtrsim n^{\frac{2r+\beta}{4(d+1+\beta)}} \mathbb{I}_{\beta \leq 2} + n^{\frac{2r+\beta}{2(2d + 3\beta)}} \mathbb{I}_{\beta > 2}$\\
    depth(L) &
    $L_1 \geq 2\left\lceil\log_2(r+ \lfloor \beta \rfloor)\right\rceil+2$ & $L_2 \geq 2\left\lceil\log_2(r+ \lfloor \beta \rfloor)\right\rceil+2$ \\
    norm constraint(K) &
    $K_1 \asymp n^{\frac{d+1}{2(d+\beta+1)}} \mathbb{I}_{\beta \leq 2} + n^{\frac{d+1}{2d + 3\beta}} \mathbb{I}_{\beta > 2}$  &
    $K_2\asymp n^{\frac{r+1}{2(d+1+\beta)}} \mathbb{I}_{\beta \leq 2} + n^{\frac{r+1}{2d + 3\beta}} \mathbb{I}_{\beta > 2}$
    \\
    \bottomrule
  \end{tabular}
\end{table}

 Then the  ERM solution $(\hat{\F}, \hat{\h})$ from \eqref{eq:erm} satisfy:
\begin{equation}
    \label{ineq: main result}
    \begin{aligned}
            &\mE[R(\hat{\F}, \hat{\h})-R(\F^{*},\hstar)] 
           \precsim \tlO \left(n^{-\frac{\beta}{2(d+1+\beta)}} \mathbb{I}_{\beta \leq 2} + n^{-\frac{\beta}{2d + 3\beta}} \mathbb{I}_{\beta > 2} \right).
    \end{aligned}
\end{equation}

\end{theorem}
\begin{remark}
\label{remark about main resultup}
    Theorem \ref{main resultup} illustrates that, with an appropriate configuration of neural networks in terms of width, depth, and norm constraints, it is possible to design ERM solutions for upstream representation learning  that exhibit a rate for nonparametric estimation. 

{A key feature of Theorem \ref{main resultup} is that, by appropriately controlling the norm constraint on the network, the rate can be achieved with only the restriction on the lower bound of the depth and width for $\cG$. }
Hence, we can attain the consistency results within an over-parameterized network class, representing a novel finding in the realm of deep transfer learning, to the best of our current understanding.

\citet[Theorem 4]{tripuraneni2020theory} presented a generalization bound for multitask transfer learning. However, their analysis is based on  a well-specified neural network model, where the target function is also a network, without accounting for the approximation error as we have done.

\end{remark}

Since the true vector $\F^*_{s}, s\in \mathcal{S}$ is sparse, we also want to 
derive the consistency for the estimation error. Hence, we introduce the following  assumptions. 
\begin{assumption}
    \label{Compatibility condition} (Regularity conditions on estimation).
    \begin{itemize}
    \item[(A4)] 
    Denote $\hat{\H}_{i} = \frac{1}{\sqrt{n}}(\hat{\h}(X_{i,1}),  \ldots, \hat{\h}(X_{i,n_{i}}) )$. 
    We assume that $\|\hat{\H}_{i}\|_{2} > \sqrt{B_{1}}$
    for all the $i \in [p]$, where $ B_{1} > 0$ is a constant; 
    \item[ (A5)]
    The ERM solution $\hat{\h}$ from \eqref{eq:erm} satisfied
    $\mE[\|\hstar -\hat{\h}\|_{L_{1}(\mathbb{P}_{\X})}] \asymp \mE[\cW_{1}(\hat{\h}(\X), \hstar(\X))]$.
    \end{itemize}
\end{assumption}
\begin{remark}
Assumption (A4) is  mild due to the condition $r \ll n$, ensuring that each $\hat{\H}_{i}^{T}$  easily becomes 
full column rank.
Assumption (A5) is a technical condition which 
 requires  that there exist a constant $B_{2}$ such that 
 $\mE[\|\hstar -\hat{\h}\|_{L_{1}(\mathbb{P}_{\X})}] \leq B_2 \mE[\cW_{1}(\hat{\h}(\X), \hstar(\X))]$, given another  direction of the inequality holds trivially.
\end{remark}

\begin{theorem}[Estimator error  of $\hat{\F}$]
    \label{Theorem: sparsity of erm solutions}
   Under Theorem \ref{main resultup} and Assumption \ref{Compatibility condition},  for any $\delta > 0$, there exits constant $\gamma > 0$ and a constant $C(\gamma, B, B_{1}, B_{2}, r, p, \delta)$ such that if we set $\mu = \frac{\gamma}{2Brp} (n^{-\frac{\beta}{2(d+1+\beta)}} \mathbb{I}_{\beta \leq 2} + n^{-\frac{\beta}{2d + 3\beta}} \log n \mathbb{I}_{\beta > 2})$, then with probability at least $1 - 2\delta$ over the draw of dataset $\mathcal{D}$, 
\[\|\hat{\F} -  \F^{*}\|_{2} \leq  C(\gamma, B, B_{1}, B_{2}, r, p, \delta) (n^{-\frac{\beta}{2(d+1+\beta)}} \mathbb{I}_{\beta \leq 2} + n^{-\frac{\beta}{2d + 3\beta}} \log n \mathbb{I}_{\beta > 2}).\]
\end{theorem}

\subsection{Results for downstream prediction}
\label{sec: theoretic result for downstream prediction}
First, we give the   assumptions on the  smoothness of target  function, and  transferable condition.
\begin{assumption}(Regularity conditions for  down-stream task).
\label{assump: downstream smooth}
(A6) 
Each entries of $q^*$ is contained in $\cH^{\beta}([-1, 1]^d, B)$. {$\|\F_{T}^{*} \|_{\infty} \leq B$} and $\|\A^{*} \|_{F} \leq B$; 
(A7)
$\ell( \cdot, y)$
is  Lipschitz with respect to $y \in \mathcal{Y}$ and the Lipschitz constant is bounded by $L_0$,  and $|\ell ( \cdot, \cdot)|\leq B_3$;  
(A8)
$\cW_1(\X, \X_T) \leq \omega$ with 
$\omega  \asymp  \left(n^{-\frac{1}{2}} \mathbb{I}_{\beta \leq 2} + n^{-\frac{\beta+d+1}{2d + 3\beta}} \mathbb{I}_{\beta > 2}\right)$; 
(A9)
Consider $\cH_{\eta} = \{h \in \cH | \mE[\cW_1(\h(\XT), \hstar(\XT))] \leq \eta, \eta  \precsim \tlO (n^{-\frac{\beta}{2(d+1+\beta)}} \mathbb{I}_{\beta \leq 2} + n^{-\frac{\beta}{2d + 3\beta}} \mathbb{I}_{\beta > 2} ) \}. $\\ Suppose
$    \sup_{ \h \in \cH_{\delta}}
     \mE[\mathbb{E}_{\YT}[\cW_1(\F_{T}^{*}\h(\XT)|\YT, \F_{T}^{*}\hstar(\XT)|\YT)]]   \leq \nu  \mE[\cW_1(\h(\XT), \hstar(\XT))], $
where $\nu > 0$ is a constant.
\end{assumption}



\begin{remark}
Conditions (A6)-(A7) are standard 
requirements in the context of nonparametric supervised learning. Condition (A8) characterizes the similarity between the upstream and downstream tasks, a prerequisite for transfer feasibility. Let $\Pr_{\X}$ and $\Pr_{T}$ denote the marginal distributions of $\X$ and $\XT$ in the upstream and downstream tasks, respectively.
In previous work, specifically in \cite{tripuraneni2020theory}, the condition $\Pr_{\X} = \Pr_{T}$ or $\cW_1(\X,\XT)=0$ was deemed necessary. 
In contrast, our assumption relaxes this requirement, allowing $\cW_1(\X,\XT)$ to be small, indicating that a minor shift in input data is acceptable.
Condition (A9) constitutes a technical requirement on local conditional 
distribution matching. Specifically, it implies that when the distribution of $\h(\XT)$ closely approximates that of $\hstar(\XT)$, the average conditional Wasserstein distance between the representations $\h(\XT)$ and $\hstar(\XT)$ is also small.
In a recent study on domain adaptation, a similar, albeit more stringent, assumption known as conditional invariant component is employed, as outlined in \cite{wu2023prominent}.
\end{remark}

 \begin{theorem}[Excess risk bound for downstream task]
\label{fine-tuning result}
Under the Assumptions \ref{assump: regularity}, \ref{assump: downstream smooth},  we plug the learned  representation $\hat{\h}$  described in  Theorem \ref{main resultup} into (\ref{finetune erm}), and set the network class  $\mathcal{Q} = \mathcal{N N}_{d^*,1}(W_3, L_3, K)$ satisfying $W_3 \gtrsim m^{\frac{2d^* + \beta}{4(d^*+1+\beta)}}, L_3 \geq  2\left\lceil\log _2(d^*+s)\right\rceil+2, K \asymp m^{\frac{d^*+1}{2(d^*+1+\beta)}}.$
Then ERM solution 
$\hat{\F}_{T}$ and $\hQ$ enjoy the following 
\begin{equation*}
    \label{ineq: fine-tuning result}
    \begin{aligned}
            &\mE [\cL(\hat{h}, \hQ, \hat{\F}_{T}) - \cL(\hstar, \Qstar , \F^{*}_{T})]  
             \leq  \tlO \left(n^{-\frac{\beta}{2(d+1+\beta)}} \mathbb{I}_{\beta \leq 2} + n^{-\frac{\beta}{2d + 3\beta}} \mathbb{I}_{\beta > 2}\right) + \mathcal{O} \left(m^{-\frac{\beta}{2(d^*+1+2\beta)}}\right).
    \end{aligned}
\end{equation*}
\end{theorem}
\begin{remark}
\label{remark about comparison of theory results}
The above Theorem \ref{fine-tuning result}
 demonstrates that end-to-end  prediction accuracy for the downstream task  is $\mathcal{O}(m^{-\frac{\beta}{2(d^*+1+2\beta)}})$ if the error in upstream training  is negligible due to the huge sample size $n$.  
\end{remark}

\begin{remark}
\label{three classes}
In the challenging scenario of ``no transfer" where $\F_T^{*} = 0$  and $d^*=d$, the convergence rate goes back to  $\mathcal{O}(m^{-\frac{\beta}{2(d+1+2\beta)}})$. This rate aligns with the convergence rate achieved using ERM for learning the $d$-dimensional function $Q^*$ in classification \citep{shen2022approximation} and is slower than the minimax rate observed in regression \citep{schmidt2019deep,nakada2020adaptive,farrell2021deep,chen2022nonparametric,kohler2022estimation,fan2023factor,bhattacharya2023deep,jiao2023deep}.
{Again, same with Remark \ref{remark about main resultup} on Theorem \ref{main resultup}}, 
the convergence rate established in Theorem \ref{fine-tuning result} is applicable to deep learning models characterized by over-parametrization. 
In the case  of ``partial transfer", where $\F_T^{*} \neq  0$ and $Q^*(\cdot) = q^*(\A^* \cdot)$ with $d^*< d$, the convergence rate demonstrates improvement from $\mathcal{O}(m^{-\frac{\beta}{2(d+1+2\beta)}})$ to $ \mathcal{O}(m^{-\frac{\beta}{2(d^{*}+1+2\beta)}})$, benefiting from the knowledge embedded in the upstream representation.
In the case of ``complete transfer" where $\Qstar =0$, the convergence rate experiences further enhancement from $\mathcal{O}(m^{-\frac{\beta}{2(d+1+2\beta)}})$ to $\mathcal{O}(\frac{1}{\sqrt{m}})$. This enhancement is elaborated in detail   in the following Theorem \ref{theorem: special case of sparse erm during fine-tuning}. Before this,  we first give some assumptions. 
\end{remark}

\begin{assumption}
    \label{Compatibility condition for downstream} (Regularity conditions on estimation error for downstream task). 
    \begin{itemize}
    \item[(A10)] 
    Denote $\hat{\H}_T = \frac{1}{\sqrt{m}}(\hat{\h}(\X_{T,1}),  \ldots, \hat{\h}(\X_{T, m}) )$. We assume that $\|\hat{\H}_T\|_{2} > \sqrt{B_{1}}$; 
    \item[(A11)] 
    $\mE[\|\hstar -\hat{\h}\|_{L_{1}(\mathbb{P}_{\XT})}] \asymp \mE[\cW_{1}(\hat{\h}(\XT), \hstar(\XT))]$.
    \end{itemize}
\end{assumption}

\begin{theorem} (Estimation error for ``complete transfer")
    \label{theorem: special case of sparse erm during fine-tuning}
   Under the Assumptions \ref{assump: regularity}, \ref{assump: downstream smooth} and \ref{Compatibility condition for downstream},  
   for any $\delta > 0$, there exits constant $\gamma_{1} > 0$, constant $C_{T,1}(\gamma_{1}, B, B_{2}, r, \delta)$ and $C_{T,2}(\gamma_{1}, B, B_{2}, r, \delta)$ such that if we set $\chi = \frac{1}{\sqrt{m}} $ in the least squares
   regression for complete transfer, 
   then with probability at least $1 - 2\delta$ over the draw of dataset $\cT$
\begin{align*}
     & \mE_{\cD}[\|\F^{*}_{T} - \hat{\F}_{T}\|_{2}] \\
     &~~~~ \leq  C_{T, 1}(B, r, \delta)m^{-\frac{1}{2}} +C_{T,2}(\gamma_{1}, B, B_{2}, r, \delta)  (n^{-\frac{\beta}{2(d+1+\beta)}} \mathbb{I}_{\beta \leq 2} + n^{-\frac{\beta}{2d + 3\beta}} \log n \mathbb{I}_{\beta > 2}).
\end{align*}
It shows the sparsity of $\hat{\F}_{T}$  when 
assumption \ref{assump: downstream smooth} and \ref{Compatibility condition for downstream} holds and 
m is considerable large. 
\end{theorem}
\begin{remark}
   Notably, in the binary classification problem, we can change the loss into logistic loss and utilize its local strongly convexity to get the similar result. 
\end{remark}

\section{Experiment on four real classification datasets and one regression dataset}
\label{sec: Experiment}
We substantiate the effectiveness of our method through experimental validation using four image
classification datasets and one regression dataset. 
The following tables all omit the percentage sign (\%) for the sake of simplicity when presenting the classification accuracy. More details on  hyperparameter setting are shown in the supplementary material.

\textbf{Datasets.}
Among the multi-domain image datasets, PACS~\citep{li2017deeper}, VLCS~\citep{fang2013unbiased}, TerraIncognita~\citep{beery2018recognition} and OfficeHome~\citep{venkateswara2017deep} are widely employed. VLCS represents a challenging dataset that is composed of 10,700 images. It encompasses four domains, namely VOC2007 (V), LabelMe (L), Caltech (C), and Sun (S), along with 5 classes. PACS  is a dataset comprising 9,900 images, which has 7 classes and 4 domains. TerraIncognita is a geographic dataset consisting of images that cover a wide variety of environments and objects. It consists of 4 domains and 24,700 images. OfficeHome contains images gathered from four distinct domains: Art (A), Clipart (C), Product (P), and Real-world (R), along with corresponding annotations for classification tasks in office environments. It comprises 15,500 images, which are partitioned into 65 object categories. All of these classification datasets are constituted of image data with four distinct domains.

The regression dataset employed is the tool wear dataset from the 2010 PHM Challenge \citep{li2009fuzzy}, which is among the rare common datasets utilized for multi-domain regression. This dataset is composed of labeled data spanning three domains, specifically C1, C4, and C6.
The original dataset contains sensing signals sampled at a frequency of 50 kHz. These signals encompass cutting force signals, vibration signals in the x, y, and z directions, as well as acoustic emission signals that are generated during the processing stage. We use cutting force signals in three directions as input and the average wear values of the three blades for each tool as output. 


\textbf{Model selection.}
In the classification task, we partition data from different domains into training and validation sets in an 8:2 ratio. The transfer process involves the following steps: withholding data from one domain, merging the training sets of the remaining domains to construct the upstream training set and the validation sets to construct the validation set accordingly. For each domain, we vary the hyperparameter $\lambda$, $\tau$ and $\mu$, selecting the model with the highest accuracy on the validation set we construct from the upstream training set. During the downstream task, training is conducted utilizing the training set from the reserved domain, with adjusting the hyperparameter $\kappa, \chi,\zeta$ and $d^*$.
The model with the highest training accuracy is chosen, and its accuracy on the withheld domain's validation set is considered as the final accuracy.
In the regression task, the distinction lies in partitioning data from various domains into training and validation sets at a 9:1 ratio. Meanwhile, Mean Absolute Error (MAE) Root and Mean Square Error (RMSE) 
are used to evaluate performance of methods.

\textbf{Network structure}
The network structure of  $\h$ in upstream  classification task is deinfed as $\h= \mathrm{LinMLP}\circ \mathrm{FeExtResNet50},$ where
$\mathrm{LinMLP}$ is a fully connected layer, and $\mathrm{FeExtResNet50}$  is the feature extractor of ResNet50 \citep{he2016deep} pretrained on the Imagenet \citep{russakovsky2015imagenet} dataset.
%
The discriminator network for $\cG$ in upstream
is implemented as a 3-layer fully connect ReLU network.
We use EfficientNet's feature extractor \citep{tan2019efficientnet}, pretrained on the Imagenet dataset as the additional feature extractor $Q_{q, \A}$.
For the regression task, we replace the feature extractor ResNet50 and EfficientNet with convolution  neural networks of small size.

\textbf{Comparison with other works.}
Two classes were selected from the PACS dataset to compare our method with the KNN-based approach proposed by \cite{cai2021transfer}, as their method was designed specifically for binary classification.
The comparison results are presented in Table~\ref{The accuracy comparad with KNN}. The reason why the  KNN-based method performs  poor is that it  is inherently unsuitable for handling high-dimensional data.

\begin{table}[ht]
\centering
\caption{The accuracy of two selected classes from four domains of PACS}
\label{The accuracy comparad with KNN}
\begin{tabular}{lcccc}
\toprule
Method/Domain & Art & Cartoon & Photo & sketch \\
\midrule
KNN-based & 62.7 & 67.5 & 57.7 & 49.3 \\
Ours& 96.1 & 97.1 & 98.7 & 92.7 \\
\bottomrule
\end{tabular}
\end{table}

\begin{table}[ht]
\centering
\caption{The accuracy of different training methods on typical datasets}
\label{The accuracy of different training methods on typical datasets}
\resizebox{1.0\textwidth}{!}{
\begin{tabular}{lccccc}
\hline
Algorithm &
  \multicolumn{1}{l}{\textbf{TerraInc}} &
  \multicolumn{1}{l}{\textbf{PACS}} &
  \multicolumn{1}{l}{\textbf{VLCS}} &
  \multicolumn{1}{l}{\textbf{Office-home}} &
  \multicolumn{1}{l}{\textbf{Avg}} \\ \hline
ERM~\citep{vapnik1998statistical}    & 87.1~ $\pm$ ~0.1 & 94.1~ $\pm$ ~0.3 & 82.0~ $\pm$ ~0.6 & 79.9~ $\pm$ ~0.3 & 85.9~ $\pm$ ~0.3 \\
DRO~\citep{sagawa2019distributionally}    & \textbf{87.5~ $\pm$ ~0.5} & 94.4~ $\pm$ ~0.0 & 82.2~ $\pm$ ~0.5 & 80.9~ $\pm$ ~0.5 & 85.7~ $\pm$ ~0.2 \\
IRM~\citep{arjovsky2020invariant}    & 82.2~ $\pm$ ~0.9 & 93.6~ $\pm$ ~0.1 & 82.0~ $\pm$ ~0.4 & 76.2~ $\pm$ ~0.4 & 83.5~ $\pm$ ~0.3 \\
DANN~\citep{ganin2016domain}   & 87.1~ $\pm$ ~0.2 & 93.7~ $\pm$ ~0.4 & 82.1~ $\pm$ ~0.9 & 79.4~ $\pm$ ~0.1 & 85.6~ $\pm$ ~0.3 \\
SagNet~\citep{nam2021reducing} & 87.3~ $\pm$ ~0.3 & 94.0~ $\pm$ ~0.5 & 81.7~ $\pm$ ~0.3 & 79.5~ $\pm$ ~0.4 & 85.7~ $\pm$ ~0.1 \\
 Ours &84.2~ $\pm$ ~0.2 &\textbf{95.7~ $\pm$ ~0.4} & \textbf{85.3~ $\pm$ ~0.1} & \textbf{84.0~ $\pm$ ~0.3}& \textbf{87.3~ $\pm$ ~0.1}

\\
\hline
\end{tabular}
}
\end{table}

While our research setting diverges from that of Domain Adaptation, Domain Generalization, and Multi-task Learning as discussed in \cite{farahani2021brief}. And it also differs from single-source domain transfer learning, as discussed in \cite{zhuang2020comprehensive}. To demonstrate the superiority of our learning approach, we conducted comparisons among several renowned Domain Adaptation methods that center around learning effective representations, along with the classical ERM algorithm,, as outlined in Table \ref{The accuracy of different training methods on typical datasets}, utilizing the code from DomainBed~\citep{gulrajani2020search}. Despite these methods do not impose any constraints on downstream training, the test results demonstrate the ascendancy of our approach to some extent. Further implementation details regarding the comparison are provided in the  supplementary material.

\textbf{Ablation study and corresponding results.}
To verify the effectiveness of additional constraints and architecture in our framework. We present the test results in Table \ref{The accuracy of different training methods on typical datasets and their domains(new)} for the following training methods.
ERM-UD (ERM on both upstream and downstream): Training a network without considering task diversity, transferring the feature extractor to the downstream task, and training a linear classifier while freezing the feature extractor.
Transfer Invariant Representation (TIR): Adopting and freezing the feature extractor using our upstream training method, followed by training a linear classifier at the downstream task without the $Q$ network.
ERM-D (ERM only at downstream): Solely utilizing the pretrained EfficientNet to train the downstream task.
Without Independence (WI): Adopting the feature extractor through our upstream training method and incorporating the pretrained EfficientNet while training at the downstream task without imposing any independence constraint between the output of EfficientNet and the transfer featurizer.

\begin{longtable}{llllll}
\captionsetup{width = 0.9\textwidth}
\caption{The accuracy of ablation study on typical datasets and their domains}
\label{The accuracy of different training methods on typical datasets and their domains(new)} \\
 \hline
\textbf{Dataset/Alg} & \multicolumn{5}{l}{\textbf{Accuracy on domain}}     \\ \hline
\textbf{officehome}        & A     & C    & P    & R    & \textbf{Avg}           \\
ERM-UD               &71.4~ $\pm$ ~1.0 & 73.3~ $\pm$ ~0.1 & 90.3~ $\pm$ ~0.1 & 83.8~ $\pm$ ~0.6 & 79.7~ $\pm$ ~0.1          \\
TIR               & 71.1~ $\pm$ ~0.6  & 72.9~ $\pm$ ~0.6 & 89.8~ $\pm$ ~0.7 &  82.7~ $\pm$ ~0.3& 79.1~ $\pm$ ~0.1          \\
ERM-D                & 51.2~ $\pm$ ~1.9  & 60.4~ $\pm$ ~1.2 & 83.1~ $\pm$ ~0.3 &  78.8~ $\pm$ ~0.5& 68.4~ $\pm$ ~0.2         \\
WI                &  70.6~ $\pm$ ~0.7 & 73.3~ $\pm$ ~0.1 & 89.7~ $\pm$ ~0.4 & 83.2~ $\pm$ ~0.5 &  79.2~ $\pm$ ~0.3        \\
Ours              & \textbf{77.5~ $\pm$ ~1.1}  & \textbf{78.4~ $\pm$ ~0.3} & \textbf{93.3~ $\pm$ ~0.1} &\textbf{86.7~ $\pm$ ~0.5} & \textbf{84.0~ $\pm$ ~0.3} \\ \hline
\textbf{PACS}              & A     & C    & P    & S    & \textbf{Avg}           \\
ERM-UD               & 92.3~ $\pm$ ~0.7  & 93.7~ $\pm$ ~0.6 & 99.0~ $\pm$ ~0.3& 88.8~ $\pm$ ~0.5 & 93.5~ $\pm$ ~0.2          \\
TIR               &  94.2~ $\pm$ ~0.2 & 94.4~ $\pm$ ~0.1 & 99.2~ $\pm$ ~0.2 & 89.6~ $\pm$ ~0.4 & 94.3~ $\pm$ ~0.1        \\
ERM-D                &  89.5~ $\pm$ ~0.2 & 88.4~ $\pm$ ~0.3 & 99.1~ $\pm$ ~0.3 & 90.8~ $\pm$ ~0.3&  92.0~ $\pm$ ~0.2         \\
WI                &  94.2~ $\pm$ ~0.2 & 93.6~ $\pm$ ~0.2 & 98.9~ $\pm$ ~0.5 & 89.3~ $\pm$ ~0.6 &   94.0~ $\pm$ ~0.1        \\
Ours              & \textbf{95.0~ $\pm$ ~0.4}  & \textbf{96.0~ $\pm$ ~0.4} & \textbf{99.4~ $\pm$ ~0.3} & \textbf{92.6~ $\pm$ ~0.7} &
\textbf{95.7~ $\pm$ ~0.4} \\ \hline
\textbf{VLCS}              & C     & L    & S    & V    & \textbf{Avg}           \\
ERM-UD               & 99.4~ $\pm$ ~0.2  & 72.2~ $\pm$ ~2.0 & 78.2~ $\pm$ ~1.0 & 84.0~ $\pm$ ~0.4 & 83.5~ $\pm$ ~0.7          \\
TIR               & 99.8~ $\pm$ ~0.2  & \textbf{74.8~ $\pm$ ~1.0} & \textbf{80.6~ $\pm$ ~1.0} & 83.6~ $\pm$ ~0.9 & 84.7~ $\pm$ ~0.6          \\
ERM-D                & 100.0~ $\pm$ ~0.0  & 73.0~ $\pm$ ~1.0 & 79.8~ $\pm$ ~0.4& 86.9~ $\pm$ ~0.2 & 85.0~ $\pm$ ~0.4          \\
WI                &  99.5~ $\pm$ ~0.2 & 74.8~ $\pm$ ~0.4 & 80.2~ $\pm$ ~0.5 & 84.6~ $\pm$ ~0.6 &  84.8~ $\pm$ ~0.1      \\
Ours              & \textbf{100.0~ $\pm$ ~0.0}  & 73.9~ $\pm$ ~0.0 & 80.0~ $\pm$ ~0.1 & \textbf{87.3~ $\pm$ ~0.3} & \textbf{85.3~ $\pm$ ~0.1} \\ \hline
\textbf{TerraInc }         & 38    & 43   & 46   & 100  & \textbf{Avg}           \\
ERM-UD               & 85.2~ $\pm$ ~0.4  & 82.1~ $\pm$ ~0.6 & 74.2~ $\pm$ ~0.5 & 89.3~ $\pm$ ~0.6 & 82.7~ $\pm$ ~0.3          \\
TIR               & 83.1~ $\pm$ ~0.8  & 76.8~ $\pm$ ~0.9 & 72.9~ $\pm$ ~0.8 & 86.2~ $\pm$ ~0.7 & 79.7~ $\pm$ ~0.6          \\
ERM-D                &  85.1~ $\pm$ ~0.4 & 69.9~ $\pm$ ~0.3 & 70.6~ $\pm$ ~0.5 & 83.9~ $\pm$ ~0.3 & 77.5~ $\pm$ ~0.1       \\
WI                &  82.2~ $\pm$ ~0.6 & 76.9~ $\pm$ ~0.6 & 71.7~ $\pm$ ~0.8 & 85.8~ $\pm$ ~0.7 &  79.2~ $\pm$ ~0.3       \\
Ours              & \textbf{87.4~ $\pm$ ~0.4}  & \textbf{83.5~ $\pm$ ~0.6} & \textbf{76.3~ $\pm$ ~0.3} & \textbf{89.6~ $\pm$ ~0.4} & \textbf{84.2~ $\pm$ ~0.2} \\ \hline
\end{longtable}

\textbf{Results interpretation.}
Contrasting the results between our method and WI demonstrates the efficacy of the independence penalty term in improving accuracy for downstream tasks, which can be attributed to that the independence constraints (\ref{qind}) and (\ref{modelQ}) reduce the learning  difficulty of the downstream task.  Moreover, when comparing the outcomes of our approach with those of ERM-D, it becomes evident that the enhanced accuracy cannot be solely attributed to the information learned in $Q$.
{Furthermore, the comparison between TIR and WI reveals that the auxiliary network for extracting different information can be detrimental to accuracy when independence requirements are not imposed.} These numerical results align consistently with our observations in Theorem \ref{fine-tuning result} and Remark \ref{remark about comparison of theory results}. Specifically, the model under ``partial transfer" conditions is demonstrated to be more easily learnable, primarily attributed to not only the advantageous representation obtained from the upstream domains but also the  imposition of the independence constraint, which significantly reduces the estimation complexity in the downstream task.

\textbf{Results of regression dataset.} 
The results were shown in Table \ref{Fine-tuning accuracy on regression dataset and their domains}. The phrase ``C4, C6 -C1" in that represents using the data from domains C4 and C6 as upstream data, and the data from domain C1 as downstream data. The subsequent ``C1, C6 -C4" and ``C1, C4 -C6" have similar meanings.
\begin{table}[ht]
\centering
\caption{Fine-tuning accuracy on the Wear dataset}
\label{Fine-tuning accuracy on regression dataset and their domains}
\begin{tabular}{lrrr}
\hline
Transfer task/Algorithm & 
    C4,C6 -C1 & 
    C1,C6 -C4 & 
    C1,C4 -C6 \\ \hline
Metrics & MAE & MAE & MAE \\ 
ERM-UD & 8.68  $\pm$  1.26 & 5.54  $\pm$  0.29 & 2.14  $\pm$  0.01 \\
Ours & \textbf{1.22  $\pm$  0.06} & \textbf{1.47  $\pm$  0.19} & \textbf{1.23  $\pm$  0.11} \\ \hline
Metrics & RMSE & RMSE & RMSE \\ 
ERM-UD & 15.10  $\pm$  0.48 & 7.36  $\pm$  0.43 & 2.77  $\pm$  0.04 \\
Ours & \textbf{2.31  $\pm$  0.15} & \textbf{1.95  $\pm$  0.21} & \textbf{1.73  $\pm$  0.12} \\ \hline
\end{tabular}
\end{table}
\begin{table}[ht]
\centering
\captionsetup{width = 0.9\textwidth}
\caption{Impact of different representation dimensions on the OfficeHome}
\label{table: effitive dimension} 
\resizebox{1.0\textwidth}{!}{
\begin{tabular}{ccccccc}
\hline Domain $\backslash$ Representation dimension & 256&  512 & 1024 & 2048 & 4096 & 8192 \\ 
\hline
A &  54.8  $\pm$  1.7 &  $57.3 \pm 2.8$ & \textbf{59.4 $\pm$  1.1 } & 56.9  $\pm$  0.5  & $ 54.1\pm 1.5$ & 50.1  $\pm$ 1.6 \\
C &  51.4  $\pm$  1.4 & \textbf{58.6  $\pm$  0.9}  & 56.8  $\pm$ 1.4 & 53.5 $\pm$ 0.5 & 46.6  $\pm$  2.3 & 38.8  $\pm$  2.3 \\
P & 76.1  $\pm$  1.5 &  77.9  $\pm$  1.0 & \textbf{79.4  $\pm$  2.1} & 72.5  $\pm$  1.5 & 59.7  $\pm$  1.0 & 50.3  $\pm$ 0.7\\
R & 71.2  $\pm$  2.3 &\textbf{ 77.5  $\pm$  0.9 }& 76.8  $\pm$  0.2 & 74.4  $\pm$  1.1 & 65.0 $\pm$  1.9 & 59.7  $\pm$  3.5\\
\hline
\end{tabular}
}
\end{table}

\begin{table}[ht]
\centering
\captionsetup{width = 1.0\textwidth}
\caption{The Impact of different representation dimensions on the Regression task}
\label{table: effitive_dimension_reg} 
\resizebox{1.0\textwidth}{!}{
\begin{tabular}{ccccccc}
\hline Task $\backslash$ dimension &  & 265 & 512 & 1024 & 2048  \\
\hline
 \multirow{2}{*}{C4,C6 -C1} & MAE & 6.59  $\pm$  0.13 & \textbf{4.87  $\pm$  0.11} & $6.20 \pm 0.71$ & $10.49  \pm 0.2$ \\
 & RMSE & 9.17  $\pm$  0.12  & \textbf{5.89 $\pm$  0.11} & $6.45 \pm 0.67$ & $15.45 \pm 0.20$ \\
 \hline
  \multirow{2}{*}{C1,C6 -C4 }& MAE & 10.38  $\pm$  0.04 & \textbf{8.57 $\pm$  0.17} & $10.07 \pm 0.38$ & $10.33 \pm 0.2$  \\
& RMSE & 15.09 $\pm$  0.08 & \textbf{13.82  $\pm$  0.14} & $14.98 \pm 0.42$ & $15.47 \pm 0.04$ \\ 
 \hline
   \multirow{2}{*}{C1,C4 -C6}& MAE &  1.64  $\pm$  0.08 & \textbf{1.41  $\pm$  0.04} & $1.92 \pm 0.10$ & $2.77 \pm 0.23$  \\
 & RMSE & 2.22  $\pm$  0.04  & \textbf{1.84 $\pm$  0.01} & $2.72 \pm 0.13$ & $3.45 \pm 0.23$\\

\hline
\end{tabular}
}
\end{table}

\textbf{The effect of representation dimension.}
To determine the optimal dimension of the representation, we conducted experiments on the OfficeHome and the Wear dataset, modifying the dimensionality of the transferred representations based on the TIR training method, which consider the downstream model as \eqref{eq:targetdataregf}, with the results shown in Table \ref{table: effitive dimension}, \ref{table: effitive_dimension_reg}. By comparing the impact of different dimensional representations on downstream tasks, we observed that the accuracy of downstream tasks first increases and then decreases as the dimensionality of the transferred representation increases. This suggests that when the representation dimension is small, it fails to capture enough predictive information, while when the dimension is too large, it introduces unnecessary high dimension information, increasing the cost of estimating a high dimension function during downstream learning. These results align consistently with the result using same technique utilized in our Theorem \ref{fine-tuning result} applied to model \eqref{eq:targetdataregf}.
Therefore, we selected 1024 and 512 as the dimensionality of the transferred representations for classification tasks and regression tasks respectively .

\section{Conclusion}
\label{sec: Conclusion}
In this paper, we propose a statistical framework for transfer learning from a multi-task upstream scenario to a single downstream task. Our framework accommodates both shared and specific features, explicitly identifying the upstream features contributing to downstream tasks. We introduce a novel training algorithm tailored for transferring a suitably invariant shared representation. Convergence analysis reveals that the derived rate is adaptive to the difficulty of transferability. Empirical experiments on real classification and regression data validate the efficacy of our model and theoretical findings.

Learning a pre-trained model  with huge data set and then 
fine-tuning on specific task \citep{lee2019mixout, brown2020language, hu2021lora, sung2021training, ruckle2021adapterdrop, xu2021raise, liu2022few, jiang2022cross, zhang2023multimodal, gao2023clip} has empirically proven to be a highly effective strategy in recent years. Our framework for multi-domain transfer learning can also be interpreted along these lines, wherein learning representations from multiple domains can be viewed as pre-training, and refining downstream predictions can be seen as fine-tuning. Consequently, the theoretical results established in this paper also contribute to enhancing our understanding of pre-training and fine-tuning processes.

\spacingset{1.1}

\onecolumn

\newpage
\spacingset{1.9}
\appendix
\begin{center}{\large \bf Supplementary Material to “Deep Transfer Learning: Model Framework and Error Analysis”}
\end{center}
\if1\blind
{
\begin{center}
\author{\small Yuling Jiao,}
\author{\small Huazhen Lin,}
\author{\small Yuchen Luo,}
\author{\small Jerry Zhijian Yang}
\end{center}
\vskip12pt
} \fi

\setcounter{table}{0}
\setcounter{page}{1}

\section{Proof of error decomposition and approximation decomposition}
$\cH$ to be a function class of features mapping $\mR^d \to \mR^r$, $\cG$ to be a function class of lipschitz-constrained mapping $ \mR^r \to \mR$ and $\mathcal{Q}: \mR^{d^{*}} \to \mR$ to be a function class during fine-tuning. 

For any $i \in [p]$, $\Pr(S=i)>0$, which is a general case and means $n_{i}$ is the same order of n. Without loss of generality, we set the upper bound of random noise as 1: $| \varepsilon_{i} |\leq 1, i \in [2]$.

The proof and definitions are organized as follows: 
In the Section \ref{Proof of upstream}, we present the proof of Theorem \ref{main resultup} and Theorem \ref{Theorem: sparsity of erm solutions} in Section \ref{sec: theoretic result for upstream}. 
Section \ref{sec: proofs for fine-tuning time} is dedicated to proving the Theorem \ref{fine-tuning result} and Theorem \ref{theorem: special case of sparse erm during fine-tuning} in Section \ref{sec: theoretic result for downstream prediction}. Section \ref{Theorem cited when proof} compiles a list of the theorems cited, and Section \ref{sec: Experimental details} provides additional experimental details.

\section{Proofs of Section \ref{sec: theoretic result for upstream}}
For all the following lemmas, we omit the statement about assumptions for simplicity. The lemma proof in Section \ref{sec: Proof of error decomposition and approximation decomposition} and  \ref{sec: Statistical errors} is under the assumption we made in Theorem \ref{main resultup}. The lemma proof in Section \ref{Proof of sparsity of erm upstream} is under the assumption we made in Theorem \ref{Theorem: sparsity of erm solutions}, the lemma proof in downstream is under the assumption we made in Theorem \ref{fine-tuning result}.

\label{Proof of upstream}
\subsection{Proofs of Theorem \ref{main resultup}}
\subsubsection{Proof of error decomposition and approximation decomposition}
\label{sec: Proof of error decomposition and approximation decomposition}
Without loss of generality, we set the $\lambda=1$ and $\tau=1$.
We first presents a basic inequality for the excess risk in terms of stochastic and approximation errors.

\begin{lemma}\label{error decomposition}
\begin{equation*}
\begin{aligned}
\mE[R(\hat{\F}, \hat{\h})-R(\F^{*},\hstar)]  \leq &2\sup_{(\F, \h) \in \mathcal{B}_{pr}(B) \circ \cH} \mE[| \hR(\F, \h)-R(\F, \h)|]
\\ &+ \inf_{(\F, \h) \in \mathcal{B}_{pr}(B) \circ \cH} |R(\F, \h) - R(\F^{*}, \hstar)| .
\end{aligned}
\end{equation*}
\end{lemma}
\begin{proof}
For any $\F \in \mathcal{B}_{pr}(B)$ and any $\h \in \cH$, we have:
\begin{equation*}
    \begin{aligned}
\mE[R(\hat{\F}, \hat{\h})-R(\F^{*},\hstar)] 
=& \mE[R(\hat{\F}, \hat{\h}) - \hR(\hat{\F}, \hat{\h}) + \hR(\hat{\F}, \hat{\h}) - \hR(\F, \h)  + \hR(\F, \h) - R(\F,\h) \\
&+ R(\F,\h) - R(\F^{*},\hstar)]\\
 \leq& \mE[R(\hat{\F}, \hat{\h}) - \hR(\hat{\F}, \hat{\h})  + \hR(\F, \h) - R(\F,\h) + R(\F,\h) - R(\F^{*},\hstar)]\\
  \leq & 2\sup_{(\F, \h) \in \mathcal{B}_{pr}(B) \circ \cH} \mE[| \hR(\F, \h)-R(\F, \h)|]
+ \mE[R(\F,\h) - R(\F^{*},\hstar)].
    \end{aligned}
\end{equation*}
The term $\hR(\hat{\F}, \hat{\h}) - \hR(\F, \h) < 0$ by definition. Due to the randomness of $\F$ and $\h$, we can take the infimum of both the LHS and the RHS of the inequality. Thus,
\begin{equation*}
\begin{aligned}
\mE[R(\hat{\F}, \hat{\h})-R(\F^{*},\hstar)] \leq &  2\sup_{(\F, \h) \in \mathcal{B}_{pr}(B) \circ \cH} \mE[| \hR(\F, \h)-R(\F, \h)|] \\
 &+ \inf_{(\F, \h) \in \mathcal{B}_{pr}(B) \circ \cH} |R(\F, \h) - R(\F^{*}, \hstar)| \negmedspace . \\
\end{aligned}
\end{equation*}

\end{proof}

\begin{lemma}
\begin{equation}\label{approximation errors one}
\inf_{(\F, \h) \in \mathcal{B}_{pr}(B) \circ \cH} |R(\F, \h) - R(\F^{*}, \hstar)| \leq |R(\F^{*}, \h) - R(\F^{*}, \hstar)| \leq (4B^{2} +32p +1)\|\h-\hstar\|_{\infty}.
\end{equation}
\end{lemma}
\begin{proof}
\begin{equation*}
\begin{aligned}
|R(\F^{*}, \h) - R(\F^{*}, \hstar)| 
& \leq | \mathrm{E}[(Y_S-\F_S^{*}\h(\X_S))^2] +\cV [ \mathrm{D}, \h(X)]  - \mathrm{E}[(Y_S-\F_S^{*}\hstar(\X_S))^2] \\
& \quad -\cV [ \mathrm{D},\hstar(\X)]| + \cW_1(\h(X), \cU_{r})- \cW_1(\hstar(\X), \cU_{r}) \\
& \leq  4B\mathrm{E}[|\F_{S}^{*}(\h(X)) - \F_{S}^{*}(\hstar(\X))|]  + |\mathcal{V} [ \mathrm{D}, \h(X)] -\mathcal{V} [ \mathrm{D}, \hstar(\X)]| \\
& \quad +  \cW_1(\h(X), \hstar(\X))\\
& \leq 4B^{2} \| \h-\hstar \|_{\infty}+ |\mathcal{V} [ \mathrm{D}, \h(X)] -\mathcal{V} [ \mathrm{D}, \hstar(\X)]|  + \cW_1(\h(X), \hstar(\X))\\
& \leq (4B^{2} +32p+1)\|\h-\hstar\|_{\infty}.
\end{aligned}
\end{equation*}
Here we utilize the metric property of Wasserstein distance and propery of covariance distance in \citet[p16,Remark 3]{szekely2007measuring}.
\end{proof}

\subsubsection{Statistical errors}
\label{sec: Statistical errors}
We can decompose the analysis of statistical errors into three components by the loss we defined.
\begin{equation}
\label{train decompose}
\begin{aligned}
&\mE[\sup_{(\F, \h) \in \mathcal{B}_{pr}(B) \circ \cH} | \hR(\F, \h)-R(\F, \h)|] 
\\ \leq & \mE[\underbrace{\sup_{(\F, \h) \in \mathcal{B}_{pr}(B) \circ \cH} |\mE[(Y_S-\F_S\h(\X_S))^2] - \frac{1}{n} \sum_{i=1}^{n} (Y_{S_{i},i} - \F_{S_{i}}(\h(x_{S_{i}, i})))^2 |}_{\cL_1} 
\\  &+  \underbrace{\mE[\sup_{\h \in \cH}| V[\h(X),D] - \hV[\h(X),D]|]}_{\cL_2} + \underbrace{\mE[\sup_{\h \in \cH}| \hat{\cW_1}(\h(\X), \cU_{r})-  \cW_1(\h(\X), \cU_{r})|]}_{\cL_3}.
\end{aligned}
\end{equation}

We decompose this part of the error into three components and analyze them separately. The processing for bounding the statistical error of $\cL_1$ and $\cL_3$ is almost standard \citep{anthony2009neural, mohri2018foundations}. In the context of analysis restricted to $\cL_2$, we leveraged the methodological techniques concerning the estimation error of the U-statistic as presented in \cite{huang2024deep}.

In the following we will bounding the three term in \eqref{train decompose} one by one. 

\begin{lemma}\label{cL_1}
\begin{equation}
\cL_1 \leq 8B \mE[\sup_{(\F, \h) \in \mathcal{B}_{pr}(B) \circ \cH}  | \frac{1}{n}\sum_{i=1}^{n}\sigma_i \F_{S_{i}}\h(\X_{S_{i},i})|] + \frac{2B^2}{\sqrt{n}} .
\end{equation}
\end{lemma}
\begin{proof}
Consider ${X'_{S'_i, i} , Y'_{S'_i, i}, S'_i}_{i=1}^{n}$ be a ghost sample of \textit{i.i.d} random variables distributed as $\Pr(\x ,y, s)$ and independent of dataset $\cD$. $\sigma_{i}$ is \textit{i.i.d} random variables distributed as Rademacher distribution $Rad(\frac{1}{2})$. And $\sigma_{i}$ is independent of all other randomness.
\begin{equation*}
\begin{aligned}
\cL_1=& \mE[\sup_{(\F, \h) \in \mathcal{B}_{pr}(B) \circ \cH} |\mE[(Y_S-\F_S\h(\X_S))^2] -\frac{1}{n} \sum_{i=1}^{n} (Y_{S_{i},i} - \F_{S_{i}}(\h(\x_{S_{i}, i})))^2 |]  \\
\leq  & \mE[\sup_{(\F, \h) \in \mathcal{B}_{pr}(B) \circ \cH} |\frac{1}{n}\sum_{i=1}^{n}(Y'_{S'_i,i} - \F_{S'_{i}}\h(\X'_{S'_i,i})))^2 - \frac{1}{n}\sum_{i=1}^{n}(Y_{S_i,i} - \F_{S_{i}}\h(\X_{S_{i},i}))^2|]\\
 \leq & \mE[\sup_{(\F, \h) \in \mathcal{B}_{pr}(B) \circ \cH} |\frac{1}{n}\sum_{i=1}^{n}\sigma_i(Y'_{S'_i,i} - \F_{S'_{i}}\h(\X'_{S'_i,i})))^2 - \frac{1}{n}\sum_{i=1}^{n}\sigma_i(Y_{S_i,i} - \F_{S_{i}}\h(\X_{S_{i},i}))^2|] \\
 \leq & 2 \mE[\sup_{(\F, \h) \in \mathcal{B}_{pr}(B) \circ \cH} | \frac{1}{n}\sum_{i=1}^{n}\sigma_i(Y_{S_i,i} - \F_{S_{i}}\h(\X_{S_{i},i}))^2 |] \\
 \leq & 8B \mE[\sup_{(\F, \h) \in \mathcal{B}_{pr}(B) \circ \cH}  | \frac{1}{n}\sum_{i=1}^{n}\sigma_i \F_{S_{i}}\h(\X_{S_{i},i})|] + \frac{2B^2}{\sqrt{n}}.  \\
\end{aligned}
\end{equation*}
The last inequality is due to the \citet[Exercise 4.7c]{wainwright2019high} and the contraction principle \citep[Theorem 4.12]{ledoux2013probability}. 
\end{proof}

As for the upper bound of $\cL_2$, We need an alternative pseudo-metric to simplify the analysis. $\forall \h \in \cH$,
define a random empirical measure  (depends on $\tilde{D}_i = (\X_{i}, S_{i}), i=1,\ldots, n$)
\[
d_{D,2}(\h, \hat{\h})= \mE_{\sigma_{i_1}, i_1 = 1,...,n}|\frac{1}{C_{n}^{4}} \sum_{1\leq i_1<  i_2<i_3<i_4\leq n} \sigma_{i_{1}}(k_{\h}-k_{\hat{\h}})(\tilde{D}_{i_{1}}, \ldots, \tilde{D}_{i_{4}})|,
\]
where $\sigma_{i_1} \sim Rad(\frac{1}{2})$ \textit{i.i.d} and $k_{\h}(\tilde{D}_{i_{1}}, \ldots, \tilde{D}_{i_{4}}) = k((\h(\X_{i_{1}}), S_{i_{1}}), \ldots, (\h(\X_{i_{4}}), S_{i_{4}}))$.

Condition on $\tilde{D}_i, i=1,\ldots, n$, let $\mathfrak{C}(\cH, d_{D,2}, \delta))$ be  the covering number \citep[page 41]{gine2021mathematical} of $\cH$ with respect  to the  empirical  distance  $d_{D,2}$ at scale of $\delta>0$. Denote $\cH_{\delta}$ as the covering set of $\cH$ with with cardinality of $\mathfrak{C}(\cH, d_{D,2}, \delta))$.

\begin{lemma}
\label{stat for U-stata}
If $\xi_{i}, i = 1,...m$ are $m$ finite linear combinations of Rademacher variables $\epsilon_j, j=1,..J$. There exists some $C_0, C_1, C_2 > 0$,
Then
\begin{equation}\label{rma}
\mE_{\epsilon_j,j=1,...J} \max _{1\leq i \leq m} |\xi_{i}| \leq C_2 (\log m)^{1/2} \max _{1\leq i \leq m}\left(\mE \xi_{i}^{2}\right)^{1 / 2}.
\end{equation}
\end{lemma}
\begin{proof}
By \citet{de2012decoupling}[Corollary 3.2.6], we know that $\xi_{i}$ is all belongs to $\mathcal{L}_{\phi_2}$ (a space of Orlicz space with $\phi_2 (x) = exp(x^2)$) for their corresponding Orlicz norm \citep[page 36]{de2012decoupling} is $C_0 \mathrm{E} [\xi_{i}^2]^{\frac{1}{2}}$. So it's still a Sub-Gaussian r.v with parameter $C_1 \mathrm{E} [\xi_{i}^2]^{\frac{1}{2}}$, and by \citet{de2012decoupling}[inequality (4.3.1)] with $\Phi(x) = \exp(\rho x)$ we can get that 
\[
\mE_{\epsilon_j,j=1,...J} \max _{1\leq i \leq m} |\xi_{i}| \leq \frac{\log m}{\rho} + \frac{\log \max _{1\leq i \leq m} \exp{\frac{\rho^2 C_1 \mathrm{E} [\xi_{i}^2]^{\frac{1}{2}}}{2}}}{\rho}.
\]
, then we can choose a  optimal $\rho$ to get the wanted result.
\end{proof}

\begin{lemma}
\label{cL_2}
There exists some $C_3, C_4, C_5, C_6 > 0$, making the following inequality satisfied.
    \begin{equation}
    \cL_2 \leq C_6 \sqrt{r}  \sqrt{\frac{{\rm Pdim}(\cH_1)}{n}},
    \end{equation}
    where $\cH_{1}$ means the neural network function class $\cN\cN_{d,1}(W_1, L_1, K_1)$ which can be seen as the projection to one-dimensional of $\cH$.
\end{lemma}

\begin{proof}
Base the definition of $\mathfrak{C}(\cH, d_{D,2}, \delta))$ and $\cH_{\delta}$, 
    \begin{equation*}
    \begin{aligned}
    \cL_2 = &\mE \mE_{\sigma_{i_1}}[\sup_{h \in \cH} |  \frac{1}{C_{n}^{4}} \sum_{1 \leq i_{1}<i_{2}<i_{3}<i_{4} \leq n} \sigma_{i_1}\bar{k}_{\h}(\tilde{D}_{i_1},\tilde{D}_{i_2}, \tilde{D}_{i_3},\tilde{D}_{i_4})|]\\
    &\leq \inf_{0< \delta < 1} \{ \delta + \mE \mE_{\sigma_{i_1}}[\sup_{\h \in \cH_{\delta}}|  \frac{1}{C_{n}^{4}} \sum_{1 \leq i_{1}<i_{2}<i_{3}<i_{4} \leq n} \sigma_{i_1}\bar{k}_{\h}(\tilde{D}_{i_1},\tilde{D}_{i_2}, \tilde{D}_{i_3},\tilde{D}_{i_4})|] \}\\
    &\leq \inf_{0< \delta < 1} \{ \delta + C_3 \frac{1}{C_{n}^{4}} \mE[(\log \mathfrak{C}(\cH, d_{D,2}, \delta))^{1/2} \max_{\h \in \cH_{\delta}} [\sum_{i_1 = 1}^n \sum_{i_2<i_3<i_4} (\bar{k}_{\h}(\tilde{D}_{i_1},\tilde{D}_{i_2}, \tilde{D}_{i_3},\tilde{D}_{i_4}))^2]^{1/2}] \} \\
    &\leq \inf_{0< \delta < 1} \{ \delta + C_3 C_4  \mE[(\log \mathfrak{C}(\cH, d_{D,2}, \delta))^{1/2} \frac{1}{C_{n}^{4}}[\frac{n(n !)^{2}}{((n-3) !)^{2}}]^{1 / 2}] \} \\
    &\leq  \inf_{0< \delta < 1} \{ \delta + 2C_3 C_4  \mE[(\log \mathfrak{C}(\cH, d_{D,2}, \delta))^{1/2}/\sqrt{n}] \} \\
    &\leq  \inf_{0< \delta < 1} \{ \delta + 2C_3 C_4  \mE[(\log \mathfrak{C}(\cH, C_5 d_{D, \infty}, \delta))^{1/2}/\sqrt{n}] \}\\
    &\leq  \inf_{0< \delta < 1} \{ \delta + 2C_3 C_4  \mE[(\log \mathfrak{C}(\cH, d_{D, \infty}, \delta/C_5))^{1/2}/\sqrt{n}] \} \\
    &\leq  \inf_{0< \delta < 1} \{ \delta + 2C_3 C_4 \sqrt{r} \mE[(\log \mathfrak{C}(\cH_1, d_{D, \infty}, \delta/C_5))^{1/2}/\sqrt{n}] \} \\
    &\leq \inf_{0< \delta < 1} \{ \delta + 2C_3 C_4 \sqrt{r} ({\rm Pdim}(\cH_1) \log ( \frac{2 C_5 e n}{\delta{\rm Pdim}(\cH_1)} ) )^{1/2}/\sqrt{n} \} \\
    &  \leq C_6 \sqrt{r}  \sqrt{\frac{{\rm Pdim}(\cH_1)}{n}},
    \end{aligned}
    \end{equation*}
    where $d_{D, \infty}(\h_{1}, \h_{2}) = \sum_{i = 1}^{n} \frac{1}{n} \| \h_1(\x_{i}) - \h_{2}(\x_{2})\|_{\infty}$. The proof is base on the proof on \citet[Lemma 8.7]{huang2024deep}, and the second inequality is base on the lemma \ref{stat for U-stata}.

    As for the third last inequality, when estimating the covering number of the family of neural network functions, it is commonly approached by employing concepts such as pseudo dimension or VC dimension.
Denote the pseudo dimension of $\cF$ by ${\rm Pdim}(\cF)$,
by \citet[Theorem 12.2]{anthony1999}, for $n \geq {\rm Pdim}(\cF)$,
\[
\mathfrak{C}(\cF, d_{D, \infty}, \delta) \leq \Big (\frac{2eB_{\cF}n}{{\rm Pdim}(\cF) \delta}\Big)^{{\rm Pdim}(\cF)}.
\]
Here $B_{\cF}$ is the upper bound of funtions in $\cF$. However,we can establish its validity regardless of the relative magnitudes of $n$ and ${\rm Pdim}(\mathcal{B}_{pr}(B) \circ \cH)$, as demonstrated by the proof of \citet[Corollary 35]{huang2022error}.

For the last inequality we set the $\delta = C_4 \sqrt{\frac{{\rm Pdim}(\cH_1)}{n}} < 1$.
\end{proof}

\begin{lemma}
\label{cl_3}
\begin{equation}
\begin{aligned}
    \cL_3 & = \mE[\sup_{\h \in \cH}| \hat{\cW_1}(\h(\X), \cU_{r})-  \cW_1(\h(\X), \cU_{r})|] \\
    &  \leq 2 \sup_{h \in \cH^{1}} \inf_{g \in \cG} \|g-h \| + 2 \mE[\sup_{\h \in \cH, g \in \cG} |\frac{1}{n} \sum_{i=1}^{n} \sigma_i g(\h(\X_{i})) |] \\
    &  \quad + 2 \mE[\sup_{g \in \cG} |\frac{1}{n} \sum_{i=1}^{n} \sigma_i g(\xi_i)|] .
\end{aligned}
\end{equation}
\end{lemma}
\begin{proof}
Let $\overline{\cW_1}(\h(\X), \cU_{r}) := \sup_{g \in \cG_1} \frac{1}{K_2}\mathbb{E}_{\x \sim \h(\X)}[g(\X)] - \mathbb{E}_{\s \sim \cU_{r}}[g(z)]$ be the population form of $\hat{\cW_1}(\h(\X), \cU_{r})$, then we can derive the proof as followed:
    \begin{align*}
    \cL_3 &= \mE[\sup_{\h \in \cH}| \hat{\cW_1}(\h(\X), \cU_{r}) -  \cW_1(\h(\X), \cU_{r})|] \\
    & \leq \mE[\sup_{\h \in \cH}| \hat{\cW_1}(\h(\X), \cU_{r}) - \overline{\cW_1}(\h(\X), \cU_{r}) + \overline{\cW_1}(\h(\X), \cU_{r}) - \cW_1(\h(\X), \cU_{r})|] \\
    & \leq  \frac{2}{K_2} \sup_{h \in \cH^{1}} \inf_{g \in \cG} \|g-h \| + \frac{1}{K_2}\mE[\sup_{h \in \cH^{1}} | \sup_{g \in \cG}  \frac{1}{n} \sum_{i=1}^{n} g(\h(\X_{i})) - \frac{1}{n} \sum_{i=1}^{n}g(\xi_i) \\
    & \quad - \sup_{g \in \cG} \mE_{\x \sim \h(\X)}[g(\x)] - \mE_{\z \sim \cU_{r}}[g(\z)]|] \\
    & \leq  \frac{2}{K_2} \sup_{h \in \cH^{1}} \inf_{g \in \cG} \|g-h \| +   \frac{1}{K_2}\mE[\sup_{h \in \cH^{1}} \sup_{g \in \cG}  |\mE_{\x \sim \h(\X)}[g(\x)] -\frac{1}{n} \sum_{i=1}^{n} g(\h(\X_{i})) \\
    & \quad +  \frac{1}{n} \sum_{i=1}^{n}g(\xi_i) -  \mE_{\z \sim \cU_{r}}[g(\z)]|] \\
    & \leq  \frac{2}{K_2} \sup_{h \in \cH^{1}} \inf_{g \in \cG} \|g-h \| +   \frac{1}{K_2}\mE[\sup_{h \in \cH^{1}} \sup_{g \in \cG}  |\mE_{\x \sim \h(\X)}[g(\x)] -\frac{1}{n} \sum_{i=1}^{n} g(\h(\X_{i})) \\
    & \quad +  |\frac{1}{n} \sum_{i=1}^{n}g(\xi_i) -  \mE_{\z \sim \cU_{r}}[g(\z)]|] \\
    & \leq  \frac{2}{K_2} \sup_{h \in \cH^{1}} \inf_{g \in \cG} \|g-h \| +  \frac{1}{K_2}\mE[  \sup_{\h \in \cH, g \in \cG} | \mE_{\x \sim \h(\X)}[g(\x)] -\frac{1}{n} \sum_{i=1}^{n} g(\h(\X_{i})) | \\
    & \quad +  \sup_{g \in \cG} |\mE_{\z \sim \cU_{r}}[g(\z)] - \frac{1}{n} \sum_{i=1}^{n}g(\xi_i)|] \\ 
    &  \leq \frac{2}{K_2} \sup_{h \in \cH^{1}} \inf_{g \in \cG} \|g-h \| + \frac{2}{K_2} \mE[\sup_{\h \in \cH, g \in \cG} |\frac{1}{n} \sum_{i=1}^{n} \sigma_i g(\h(\X_{i})) |] \\
    &  \quad + \frac{2}{K_2} \mE[\sup_{g \in \cG} |\frac{1}{n} \sum_{i=1}^{n} \sigma_i g(\xi_i)|] .
    \end{align*}

    The first inequality is due to the metric property. The second inequality bounding the term $\overline{\cW_1}(\h(\X), \cU_{r}) - \cW_1(\h(\X), \cU_{r})$ by \citet[lemma 24]{huang2022error}. The third inequality is due to the $|\sup_{a} f - \sup_a g| \leq \sup_{a} |f-g |$.The second last and last inequality is again by applying chain and symmetrization technique as proof of lemma \ref{cL_1}.
\end{proof}

We now present the proof of our main result for upstream.
\begin{proof}[Proof of Theorem \ref{main resultup}]

Combine all the lemma proofed in Section \ref{sec: Proof of error decomposition and approximation decomposition} and \ref{sec: Statistical errors}, we can get that 
\begin{equation}
    \begin{aligned}
        \mE[R(\hat{\F}, \hat{\h})-R(\F^{*},\hstar)] 
        & \leq \frac{2B^2}{\sqrt{n}} + (4B^{2}+32p +1)\|\h-\hstar\| + \frac{2}{K_2} \sup_{f \in \cH^{1}} \inf_{g \in \cG} \|g-f \| \\
        & \quad +8B \mE[\sup_{(\F, \h) \in \mathcal{B}_{pr}(B) \circ \cH}  | \frac{1}{n}\sum_{i=1}^{n}\sigma_i \F_{S_{i}}\h(\X_{S_{i},i}) |]  +  C_6 \sqrt{r} \sqrt{\frac{{\rm Pdim}(\cH_1)}{n}}   \\
        & \quad + \frac{2}{K_2} \mE[\sup_{\h \in \cH, g \in \cG} |\frac{1}{n} \sum_{i=1}^{n} \sigma_i g(\h(\X_{i})) |]  + \frac{2}{K_2} \mE[\sup_{g \in \cG} |\frac{1}{n} \sum_{i=1}^{n} \sigma_i g(\xi_i)|]  \}. 
    \end{aligned}
    \label{excess risk  with pdim and appro}
\end{equation}


Based on \cite{bartlett2019nearly}, the pseudo dimension (VC dimension) of $\cF$ with piecewise-linear activation function can be further contained and represented by its depth $L$ and width $W$, i.e., ${\rm Pdim}(\cF)=O(W^{2}L^{2}\log(W^{2}L)$.

By ${\rm Pdim}(\cF)=O(W^{2}L^{2}\log(W^{2}L)$ for $\cF =\cN \cN (W,L)$ and Theorem \ref{thm: appro for norm-constrain NN}, \ref{thm: stata for norm-constrained nn}, \eqref{excess risk with pdim and appro} can be further deduced as followed:

    \begin{align*}
        \mE[R(\hat{\F}, \hat{\h})-R(\F^{*},\hstar)] 
        & \precsim \frac{1}{\sqrt{n}} + K_{1}^{-\frac{\beta}{d+1}}  + K_{2}^{-\frac{\beta}{r+1}} \\
        & \quad  + \frac{K_1}{\sqrt{n}} + \sqrt{r}\sqrt{\frac{(W_1L_1)^2 \log(W_1^2 L_1)\log n}{n}} \\
        & \quad  + \frac{K_1}{\sqrt{n}} +\frac{1}{\sqrt{n}}\\
        & \precsim  K_{1}^{-\frac{\beta}{d+1}} + K_{2}^{-\frac{\beta}{r+1}} +  \frac{K_{1}^{\frac{2d+\beta}{2d+2}}}{\sqrt{n}} + \frac{K_1}{\sqrt{n}} +\frac{1}{\sqrt{n}} \\
        & \precsim n^{-\frac{\beta}{2(d+1+\beta)}} \mathbb{I}_{\beta \leq 2} + n^{-\frac{\beta}{2d + 3\beta}} \log n\mathbb{I}_{\beta > 2} \\
        & \precsim \tlO \left( n^{-\frac{\beta}{2(d+1+\beta)}} \mathbb{I}_{\beta \leq 2} + n^{-\frac{\beta}{2d + 3\beta}}\mathbb{I}_{\beta > 2} \right),
    \end{align*}
where we set $K_1 \asymp n^{\frac{d+1}{2(d+1+\beta)}}$,  $K_2 \asymp n^{\frac{r+1}{2(d+1+\beta)}}$, and $W_1 \asymp K_1^{\frac{2d+ \beta}{2d+2}}$ when $\beta \leq 2 $.

When $\beta > 2 $, we set $K_1 \asymp n^{\frac{d+1}{2d + 3\beta}}$, $K_2 \asymp n^{\frac{d+1}{2d + 3\beta}}$, and $W_1 \asymp K_1^{\frac{2d+ \beta}{2d+2}}$.

\end{proof}

\subsection{Proof of Theorem \ref{Theorem: sparsity of erm solutions}}
\label{Proof of sparsity of erm upstream}
We now present the proof of result about sparsity of ERM solutions in upstream. Unlike the proof of Theorem \ref{main resultup} presented earlier, for the sake of explanatory coherence, we have positioned the statements and proofs of the necessary lemmas after the proof of the corresponding theorem in the subsequent demonstration.

\begin{proof}[Proof of Theorem \ref{Theorem: sparsity of erm solutions}]

For simplicity, we denote $n^{-t(\beta)}=n^{-\frac{\beta}{2(d+1+\beta)}} \mathbb{I}_{\beta \leq 2} + n^{-\frac{\beta}{2d + 3\beta}} \log n \mathbb{I}_{\beta > 2}$. 

By the definition of $(\hat{\F},\hat{\h})$ from \eqref{eq:erm}, we know that $\hR(\hat{\F},\hat{\h}) \leq \hR(\f_{0}, \hat{\h}) $.

\begin{equation*}
    \begin{aligned}
    \sum_{i=1}^{p} \sum_{j=1}^{n_p}\frac{1}{n} (Y_{i,j} - \hat{\F}_{i}\hat{\h}(\X_{i,j}))^2 + \mu  \sum_{i = 1}^{p} \| \hat{\F}_{i}\|_{1}  \leq 
    \sum_{i=1}^{p} \sum_{j=1}^{n_p}\frac{1}{n} (Y_{i,j} - \F^{*}_{i}\hat{\h}(\X_{i,j}))^2 + \mu  \sum_{i = 1}^{p} \| \F^{*}_{i}\|_{1}, \\
    \end{aligned}
\end{equation*}

\begin{equation}
    \begin{aligned}
    \label{primal sparsity}
    \sum_{i=1}^{p} \sum_{j=1}^{n_p}\frac{1}{n} (Y_{i,j} - \hat{\F}_{i}\hat{\h}(\X_{i,j}))^2  
     \leq  
    \sum_{i=1}^{p} \sum_{j=1}^{n_p}\frac{1}{n} (Y_{i,j} - \F^{*}_{i}\hat{\h}(\X_{i,j}))^2 + \mu  \sum_{i = 1}^{p} \| \hat{\F}_{i} - \F^{*}_{i}\|_{1} .
    \end{aligned}
\end{equation}

By the equation 
\begin{equation}
\begin{aligned}
\label{square term}
 \sum_{i=1}^{p} \sum_{j=1}^{n_p}\frac{1}{n} (Y_{i,j} - \hat{\F}_{i}\hat{\h}(\X_{i,j}))^2 & = \sum_{i=1}^{p} \sum_{j=1}^{n_p}\frac{1}{n} (Y_{i,j} - \F^{*}_{i}\hat{\h}(\X_{i,j}) +  \F^{*}_{i}\hat{\h}(\X_{i,j})- \hat{\F}_{i}\hat{\h}(\X_{i,j}))^2 \\
 & = \sum_{i=1}^{p} \sum_{j=1}^{n_p}\frac{1}{n} (Y_{i,j} - \F^{*}_{i}\hat{\h}(\X_{i,j}))^2 \\
 & + \sum_{i=1}^{p} \sum_{j=1}^{n_p}\frac{2}{n}(Y_{i,j} - \F^{*}_{i}\hat{\h}(\X_{i,j}))(\F^{*}_{i}\hat{\h}(\X_{i,j})- \hat{\F}_{i}\hat{\h}(\X_{i,j})) \\
 & +  \sum_{i=1}^{p} \sum_{j=1}^{n_p}\frac{1}{n}((\F^{*}_{i} -\hat{\F}_{i})\hat{\h}(\X_{i,j})^{2}, \\
\end{aligned} 
\end{equation}
after estimating the last two term in right part of upper equation. We can replace the left part in \eqref{primal sparsity} by \eqref{square term} to get the desired result by selecting proper $\mu$.

 Due to the lemma \ref{lemma: bounding hat h and hstar} below, by selecting $\mu  = \frac{\gamma}{2Bpr}  n^{-t(\beta)}$, the mid term of \eqref{square term} can be bounded as followed with probability at least $1-2\delta$ over the draw of dataset $\cD$ (w.h.p):
\begin{align*}
    & \sum_{i=1}^{p} \sum_{j=1}^{n_p}\frac{2}{n}(Y_{i,j} - \F^{*}_{i}\hat{\h}(\X_{i,j}))(\F^{*}_{i}\hat{\h}(\X_{i,j})- \hat{\F}_{i}\hat{\h}(\X_{i,j})) \\ 
    = & \sum_{i=1}^{p} \sum_{j=1}^{n_p}\frac{2}{n}( \F^{*}_{i}\hstar(\X_{i,j}) - \F^{*}_{i}\hat{\h}(\X_{i,j}) + \epsilon_{i,j})(\F^{*}_{i}- \hat{\F}_{i})\hat{\h}(\X_{i,j}) \\
    = &  \sum_{i=1}^{p} \sum_{j=1}^{n_p}\frac{2}{n}(\F^{*}_{i}(\hstar(\X_{i,j}) -\hat{\h}(\X_{i,j})) + \epsilon_{i,j})(\F^{*}_{i}- \hat{\F}_{i})\hat{\h}(\X_{i,j}) \\
    \leq & \sum_{i=1}^{p}\|\F^{*}_{i}- \hat{\F}_{i}\|_{1} 
    \|\sum_{j=1}^{n_p}\frac{2}{n} \F^{*}_{i}(\hstar(\X_{i,j}) -\hat{\h}(\X_{i,j}))\hat{\h}(\X_{i,j})\|_{\infty} \\
    + & \sum_{i=1}^{p}\|\F^{*}_{i}- \hat{\F}_{i}\|_{1} 
    \|\sum_{j=1}^{n_i}\frac{2}{n} \epsilon_{i,j} \hat{\h}(\X_{i,j})\|_{\infty} \\
    \overset{w.h.p}{\leq} & \sum_{i=1}^{p} \|\F^{*}_{i}- \hat{\F}_{i}\|_{1} \left((1+2B)(1+ 2 B_2)\gamma n^{-t(\beta)} +  (2B+1)\sqrt{\frac{\log\frac{1}{\delta}}{2n}} +  2\sqrt{\frac{2 \log \frac{2r}{\delta}}{n_{i}}}\right), \\
\end{align*}
where $\epsilon_{i,j}$ here is the \textit{i.i.d} sample of $\varepsilon_{1}$ and for simplicity we omit the footnote from $\epsilon_{1,i,j}$ to $\epsilon_{i,j}$.

By A(4) in the Assumption \ref{Compatibility condition}, we can get 
\[
\frac{1}{n}\sum_{i=1}^{p} \sum_{j=1}^{n_p}((\F^{*}_{i} -\hat{\F}_{i})\hat{\h}(\X_{i,j}))^{2}  =  \sum_{i=1}^{p} \|(\F^{*}_{i} -\hat{\F}_{i}) \hat{\H}_{i}\|_{2}^{2} \geq \sum_{i=1}^{p}   B_{1} \|\F^{*}_{i} -\hat{\F}_{i}\|_{2}^{2}.
\]

So the \eqref{primal sparsity} can be further deduced as followed:
\begin{align*}
    &\sum_{i=1}^{p} B_{1}\|\F^{*}_{i} -\hat{\F}_{i}\|_{2}^{2}  \overset{w.h.p}{\leq}  \mu  \sum_{i = 1}^{p} \| \hat{\F}_{i} - \F^{*}_{i}\|_{1}\\ 
       & \quad + \sum_{i=1}^{p} \|\F^{*}_{i}- \hat{\F}_{i}\|_{1} \left((1+2B)(1+ 2 B_2)\gamma n^{-t(\beta)} +  (2B+1)\sqrt{\frac{\log\frac{1}{\delta}}{2n}} +  2\sqrt{\frac{2 \log \frac{2r}{\delta}}{n_{i}}}\right),  \\
 \end{align*}
 By replacing $(2B+1)\sqrt{\frac{\log\frac{1}{\delta}}{2n}} +  2\sqrt{\frac{2 \log \frac{2r}{\delta}}{n_{i}}}$ by $\mathcal{O} \left( \frac{1}{\sqrt{n}}\right)$ for simplicity 
\begin{align*}
          ~~~~\sum_{i=1}^{p} &B_{1}\|\F^{*}_{i} -\hat{\F}_{i}\|_{2}^{2}\\ \overset{w.h.p}{\leq} &\sum_{i = 1}^{p} \Big(\big((1+2B)(1+ 2 B_2) + \frac{1}{2Bpr}\big)\gamma n^{-t(\beta)} +  \mathcal{O} \big( \frac{1}{\sqrt{n}} \big)\Big)\| \hat{\F}_{i} - \F^{*}_{i}\|_{1}  \\
\overset{w.h.p}{\leq} & \sum_{i = 1}^{p}  \sqrt{r} \Big(\big((1+2B)(1+ 2 B_2) + \frac{1}{2Bpr}\big)\gamma n^{-t(\beta)} +  \mathcal{O} \big( \frac{1}{\sqrt{n}} \big)\Big) \| \hat{\F}_{i} - \F^{*}_{i}\|_{2}. \\
\end{align*}
Then we have 
\[
  B_{1}\sqrt{\sum_{i=1}^{p} \|\F^{*}_{i} -\hat{\F}_{i}\|_{2}^{2}} \overset{w.h.p}{\leq}   \sum_{i = 1}^{p} \sqrt{rp} \Big(\big((1+2B)(1+ 2 B_2) + \frac{1}{2Bpr}\big)\gamma n^{-t(\beta)} +  \mathcal{O} \big( \frac{1}{\sqrt{n}} \big)\Big).
\]
Thus we get the wanted result which means w.p.t that the sparsity of $\hat{\F}_{i}$ is similar with $\F_{i}^{*}$.
\end{proof}

First, we state the following lemma, which is used in the proof of Lemma \ref{lemma: bounding hat h and hstar}.
\begin{lemma}
    \label{bounding the W1(hstar, hat h)}
    \[ \mE[\cW_{1}(\hat{\h}(\X^{'}), \hstar(\X^{'})) + \mu \sum_{i =1}^{p} (\|\hat{\F}_{i}\|_{1}- \| \F^{*}_{i}\|_{1})] \leq \gamma n^{-t(\beta)},
    \]
    where $\gamma > 0$ is a constant.
\end{lemma}
\begin{proof}[Proof of Lemma \ref{bounding the W1(hstar, hat h)}]
     The Theorem \ref{main resultup} shows $\mE[R(\hat{\F}, \hat{\h})-R(\F^{*},\hstar)]\precsim n^{-t(\beta)}$ and by the definition of $\hstar$, we know that 
\begin{align*}
    n^{-t(\beta)} 
    \gtrsim & \mE[(Y^{'}_S-\hat{\F_S} \hat{\h}(\X^{'}_S))^2 + \mathcal{V}[\hat{\h}(\X^{'}), S] + \cW_1(\hat{\h}(\X^{'}), \cU_{r}) - (Y^{'}_S-\F^{*}_S \hstar(\X^{'}_S))^2 \\
     &+ \mu \sum_{i =1}^{p} (\|\hat{\F}_{i}\|_{1}- \| \F^{*}_{i}\|_{1})]  \\
    \gtrsim & \mE[( \varepsilon^{'}_1 + \F^{*}_S \hstar(\X^{'}_S)- \hat{\F_S} \hat{\h}(\X^{'}_S))^2  + \cW_1(\hat{\h}(\X^{'}), \hstar(\X^{'})) - (\varepsilon^{'}_1 )^2 \\
    & + \mu \sum_{i =1}^{p} (\|\hat{\F}_{i}\|_{1}- \| \F^{*}_{i}\|_{1})]   \\
    =& \mE[2 \varepsilon^{'}_1 (\F^{*}_S \hstar(\X^{'}_S)-\hat{\F_S} \hat{\h}(\X^{'}_S))  + \cW_1(\hat{\h}(\X^{'}), \hstar(\X^{'})) + \mu \sum_{i =1}^{p} (\|\hat{\F}_{i}\|_{1}- \| \F^{*}_{i}\|_{1})] \\
    =& \mE[\cW_1(\hat{\h}(\X^{'}), \hstar(\X^{'})) + \mu \sum_{i =1}^{p} (\|\hat{\F}_{i}\|_{1}- \| \F^{*}_{i}\|_{1})]   \ \\
\end{align*}
Here ${\X^{'}_S, Y^{'}_S, \varepsilon^{'}_1}$ is the \textit{i.i.d} copy of ${\X_S, Y_S, \varepsilon_1}$. For the simplicity of calculation, we can set a constant $\gamma > 0$ big enough that $\mE[\cW_1(\hat{\h}(\X^{'}), \hstar(\X^{'})) + \mu \sum_{i =1}^{p} (\|\hat{\F}_{i}\|_{1}- \| \F^{*}_{i}\|_{1})] \leq \gamma n^{-t(\beta)}$

\end{proof}

\begin{lemma}
    For any $\delta > 0$,  $i \in [p]$, with probability at least $1- 2\delta$ we get:
    \begin{align*}
     \|\sum_{j=1}^{n_i}\frac{1}{n} \F^{*}_{i}(\hstar(\X_{i,j}) -\hat{\h}(\X_{i,j}))\hat{\h}(\X_{i,j})\|_{\infty} \leq & B(1+ 2 B_2)\gamma n^{-t(\beta)} +  B\sqrt{\frac{\log\frac{1}{\delta}}{2n}}, \\
     \|\sum_{j=1}^{n_i}\frac{1}{n} \epsilon_{i,j} \hat{\h}(\X_{i,j})\|_{\infty} \leq & (1 + 2B_{2})\gamma n^{-t(\beta)} +  \sqrt{\frac{\log\frac{1}{\delta}}{2n}} +  \sqrt{\frac{2 \log \frac{2r}{\delta}}{n_{i}}},
\end{align*} where $\gamma$ is a constant big enough.
    \label{lemma: bounding hat h and hstar}
\end{lemma}

\begin{proof}[Proof of Lemma \ref{lemma: bounding hat h and hstar}]
By A(1) in the Assumption \ref{assump: regularity}, we know that,
    \begin{align*}
    \|\sum_{j=1}^{n_i}\frac{1}{n} \F^{*}_{i}(\hstar(\X_{i,j}) -\hat{\h}(\X_{i,j}))\hat{\h}(\X_{i,j})\|_{\infty} \leq & B 
    \sum_{j=1}^{n_i}\frac{1}{n}\|(\hstar(\X_{i,j}) -\hat{\h}(\X_{i,j}))\|_{1}  \\
     \|\sum_{j=1}^{n_i}\frac{1}{n} \F^{*}_{i}(\hstar(\X_{i,j}) -\hat{\h}(\X_{i,j}))\hat{\h}(\X_{i,j})\|_{\infty} \leq & B  \sum_{i=1}^{p}\sum_{j=1}^{n_i}\frac{1}{n}\|(\hstar(\X_{i,j}) -\hat{\h}(\X_{i,j}))\|_{1},\\
    \|\sum_{j=1}^{n_i}\frac{1}{n} \epsilon_{i,j} \hat{\h}(\X_{i,j})\|_{\infty} \leq &
    \sum_{i=1}^{p}\sum_{j=1}^{n_i}\frac{1}{n} \| \epsilon_{i,j}\|_{\infty} \|\hstar(\X_{i,j}) -\hat{\h}(\X_{i,j})\|_{1}  \\
    & +   \|\sum_{j=1}^{n_i}\frac{1}{n} \epsilon_{i,j} \hstar(\X_{i,j})\|_{\infty},
    \end{align*}

    By lemma \ref{bounding the W1(hstar, hat h)}, we can select $\mu  = \frac{\gamma}{2Bpr}  n^{-t(\beta)}$ to get that $\mE[\cW_{1}(\hat{\h}(\X^{'}), \hstar(\X^{'}))] \leq  2\gamma n^{-t(\beta)}$.
    
    Due to the Theorem 3.3 in \cite{mohri2018foundations} and the Theorem \ref{thm: stata for norm-constrained nn}. For any $\delta > 0$, with probability at least $1-\delta$:
    \begin{align*}
           &\sum_{i=1}^{p}\sum_{j=1}^{n_i}\frac{1}{n}\|\hstar(\X_{i,j}) -\hat{\h}(\X_{i,j})\|_{1} - \mE[ \|\hstar(\X^{'}) -\hat{\h}(\X^{'})\|_{1}]  \\
           \leq &  2\mE[\sup_{\h \in \cH} \frac{1}{n} \sum_{i=1}^{p}\sum_{j=1}^{n_i} \sigma_i \| \h(\X_{i,j}) - \hstar(\X_{i,j}) \|_{1} ] + \sqrt{\frac{\log\frac{1}{\delta}}{2n}} \\
             \leq & 2\mE[\sup_{\h \in \cH} \frac{1}{n} \sum_{i=1}^{p}\sum_{j=1}^{n_i} \sigma_{i, j} \| \h(\X_{i,j}) - \hstar(\X_{i,j}) \|_{1} ] + \sqrt{\frac{\log\frac{1}{\delta}}{2n}} \\
            \leq & \sqrt{\frac{\log\frac{1}{\delta}}{2n}} + \frac{ 2r K_{1} \sqrt{2(L_{1}+2+\log (d+1))}}{\sqrt{n}}\\
            \leq & \gamma n^{-t(\beta)} +  \sqrt{\frac{\log\frac{1}{\delta}}{2n}}. \\
    \end{align*}
    The last inequality is due to properly selected $W_{1}, L_{1}, K_{1}$ during proof of Theorem \ref{main resultup}. Then, by Hoeffding’s inequality:
    \begin{align*}
        &  \Pr[\| \sum_{j=1}^{n_{i}}\frac{1}{n_{i}} \epsilon_{i,j}\hstar(\X_{i,j}) \|_{\infty}\geq \delta_{0}]\\
        =& \Pr[ \| \sum_{j=1}^{n_{i}}\frac{1}{n_{i}} \epsilon_{i,j}\hstar(\X_{i,j}) - \mE[ \varepsilon_{1}\hstar(\X)]\|_{\infty} \geq \delta_{0}]\\
        =&   \Pr[\exists k \in [r], | \sum_{j=1}^{n_{i}}\frac{1}{n_{i}} \epsilon_{i,j}\h^{*}_{k}(\X_{i,j}) | \geq \delta_{0}] \\
        =&   \Pr[\exists k \in [r], | \sum_{j=1}^{n_{i}}\frac{1}{n_{i}} \epsilon_{i,j}\h^{*}_{k}(\X_{i,j}) - \mE[\varepsilon_{1} \h^{*}_{k}(X)]| \geq \delta_{0}] \\
        \leq & 2 r \exp(-\frac{\delta_{0}^{2} n_{i}}{2}).\\
    \end{align*}
    Then we can set the $\delta_{0} = \sqrt{\frac{2 \log \frac{2r}{\delta}}{n_{i}}}$ here, then with probability at least $1-\delta$, we know $\| \sum_{j=1}^{n_{i}}\frac{1}{n_{i}} \epsilon_{i,j}\hstar(\X_{i,j}) \|_{\infty} \leq \sqrt{\frac{2 \log \frac{2r}{\delta}}{n_{i}}} $.
    
The by A(5) in the Assumption \ref{Compatibility condition}, with probability at least $1- 2\delta$ we get:
    \begin{align*}
     &\|\sum_{j=1}^{n_i}\frac{1}{n} \F^{*}_{i}(\hstar(\X_{i,j}) -\hat{\h}(\X_{i,j}))\hat{\h}(\X_{i,j})\|_{\infty} \\
    \leq & B \sum_{j=1}^{n_i}\frac{1}{n}\|\hstar(\X_{i,j}) - \hat{\h}(\X_{i,j})\|_{1} - B\mE[ \|\hstar(\X^{'}) -\hat{\h}(\X^{'})\|_{1}] +  B\mE[ \|\hstar(\X^{'}) -\hat{\h}(\X^{'})\|_{1}] \\
    \leq & B \sum_{j=1}^{n_i}\frac{1}{n}\|\hstar(\X_{i,j}) - \hat{\h}(\X_{i,j})\|_{1} - B\mE[ \|\hstar(\X^{'}) -\hat{\h}(\X^{'})\|_{1}] + B B_{2}\mE[\cW_{1}(\hat{\h}(\X^{'}), \hstar(\X^{'}))]  \\
     \leq & B(1+ 2 B_2)\gamma n^{-t(\beta)} +  B\sqrt{\frac{\log\frac{1}{\delta}}{2n}}. \\
    \end{align*}
And following the same way before
\begin{align*}
    \|\sum_{j=1}^{n_i}\frac{1}{n} \epsilon_{i,j} \hat{\h}(\X_{i,j})\|_{\infty} 
     \leq & \sum_{i=1}^{p}\sum_{j=1}^{n_i}\frac{1}{n} \| \epsilon_{i,j}\|_{\infty} \|\hstar(\X_{i,j}) -\hat{\h}(\X_{i,j})\|_{1}  +   \|\sum_{j=1}^{n_i}\frac{1}{n} \epsilon_{i,j} \hstar(\X_{i,j})\|_{\infty}\\
     \leq & (1 + 2B_{2})\gamma n^{-t(\beta)} +  \sqrt{\frac{\log\frac{1}{\delta}}{2n}} +  \sqrt{\frac{2 \log \frac{2r}{\delta}}{n_{i}}}.
\end{align*}
    
\end{proof}

\section{Proofs of Section \ref{sec: theoretic result for downstream prediction}}
\label{sec: proofs for fine-tuning time}
Next, we list the proof in the fine-tuning, which mainly uses the special property of training loss to induce the control of the excess error of the downstream task by upstream task. 
\subsection{Proof of the Theorem \ref{fine-tuning result}}
\begin{proof}[Proof of the Theorem \ref{fine-tuning result}]
In the following proof, for the simplicity of analysis, we incorporated the independence penalty into $\ell(\cdot, \cdot)$ leveraging its $L_{0}$-Lipschitz property.
Also, for the simplicity of notion, we denote the function class $\mathcal{Q} \circ \mathcal{B}_{dd^*}(B)$ as $\mathcal{Q}_{dd^{*}}:= \{Q_{q , \A}(\cdot)= q(\A \cdot) | q \in \cQ, \A \in \mathcal{B}_{dd^*}(B) \} $. By doing this, we can save much efforts about notation in the analysis without losing it's coherence.
\begin{equation}
    \label{error decom for finetune}
\begin{aligned}
    & \mE [\cL(\hat{\h}, \hQ, \hat{\F}_{T}) - \cL( \hstar, \Qstar , \F_{T}^{*})] \\
    \leq & \mE[\cL( \hat{\h}, \hQ,  \hat{\F}_{T}) - \hcL( \hat{\h}, \hQ , \hat{\F}_{T} ) + \hcL( \hat{\h}, \hQ,  \hat{\F}_{T} ) - \hcL( \hat{\h}, Q_{q, \A^{*}}, \F_{T}^{*}) \\ 
    & + \hcL(\hat{\h}, Q_{q, \A^{*}}, \F_{T}^{*})- \cL( \hat{\h}, Q_{q, \A^{*}}, \F_{T}^{*}) +  \cL(\hat{\h}, Q_{q, \A^{*}}, \F_{T}^{*}) - \cL(\hstar, Q_{q, \A^{*}}, \F_{T}^{*}) \\
    & + \cL(\hstar, Q_{q, \A^{*}}, \F_{T}^{*}) -  \cL(\hstar, \Qstar, \F_{T}^{*})] \\
    \leq & \underbrace{\mE[\cL(\hat{\h}, Q_{q, \A^{*}}, \F_{T}^{*}) - \cL(\hstar, Q_{q, \A^{*}}, \F_{T}^{*})]}_{\cL_4}\\
     & + \underbrace{2\mE[\sup_{(\F, Q_{q, \A}) \in \mathcal{B}_{r}(B) \circ \mathcal{Q}_{dd^{*}} } 
     \cL(\hat{\h}, Q_{q, \A},  \F) - \hcL(\hat{\h}, Q_{q, \A}, \F)|  + \inf_{q \in \mathcal{Q}} \|\cL(\hstar, Q_{q, \A^{*}}, \F_{T}^{*}) -  \cL(\hstar, \Qstar, \F_{T}^{*})]}_{\cL_5}.\\
\end{aligned}  
\end{equation}

By lemma \ref{cL_4} below, we have:
\[
\cL_4 \precsim  \tlO \left(n^{-\frac{\beta}{2(d+1+\beta)}} \mathbb{I}_{\beta \leq 2} + n^{-\frac{\beta}{2d + 3\beta}} \mathbb{I}_{\beta > 2}\right).
\]

By A(6) and A(7) in Assumption \ref{assump: downstream smooth}, Theorem \ref{thm: appro for norm-constrain NN}, \ref{thm: stata for norm-constrained nn} and following the similar process as the proof of lemma \ref{cL_1}:
\begin{align*}
         \cL_5 = & 2\mE[\sup_{(\F, Q) \in \mathcal{B}_{r}(B) \circ \mathcal{Q}_{dd^{*}} } |\cL(\hat{\h}, Q_{q, \A},  \F) - \hcL(\hat{\h}, Q_{q, \A}, \F)|  + \inf_{q \in \mathcal{Q}}\cL(\hstar, Q_{q, \A^{*}}, \F_{T}^{*}) -  \cL(\hstar, \Qstar, \F_{T}^{*})]\\
       \leq & 2 \mE[\sup_{(\F, Q) \in \mathcal{B}_{r}(B) \circ \mathcal{Q}_{dd^{*}} } | \frac{1}{m} \sum_{i=1}^{m} \sigma_i \ell(\F\hat{\h}(\X_{T,i}) + Q(\X_{T, i})), Y_{T,i})|] + L_0\inf_{q \in \mathcal{Q}} \| Q_{q, \A^{*}}- \Qstar\|_{\infty} \\
       \leq & 2 L_0\mE[\sup_{(\F, \q, \A) \in \mathcal{B}_{r}(B) \circ \mathcal{Q} \circ \mathcal{\A} } | \frac{1}{m} \sum_{i=1}^{m} \sigma_i  (\F\hat{\h}(\X_{T,i})+ \q(\A \X_{T, i})))|] + \frac{2B_{3}}{\sqrt{m}} \\
       &+ L_0\inf_{\q \in \mathcal{Q} }\|\q(\A^{*} \cdot) - \q^{*}(\A^{*} \cdot)\|_{\infty} \\
       \precsim & \frac{K+ B + 2B_{3}}{\sqrt{m}} + K^{\frac{-\beta}{d^{*}+1}} \\
       \precsim & m  ^{-\frac{\beta}{2(d^{*}+1+2\beta)}}.
\end{align*}

Add $\cL_4$ and $\cL_5$ together, we get the wanted result.
\[
 \mE [\cL( \hat{\h}, \hQ, \hat{\F}_{T}) - \cL( \hstar, \Qstar , \F_{T}^{*})]  \precsim \tlO \left(n^{-\frac{\beta}{2(d+1+\beta)}} \mathbb{I}_{\beta \leq 2} + n^{-\frac{\beta}{2d + 3\beta}} \mathbb{I}_{\beta > 2}\right)+  \mathcal{O} \left(m^{-\frac{\beta}{2(d^{*}+1+2\beta)}} \right).
\]

Usually, by directly assume the transferable condition we can bounding the risk from learning a transferable representation. In our case, it is the term $\cL_4$. And the term left, which is $\cL_5$, represents the difficulty of downstream task.

In the condition ``complete transfer" , the proof before can process almost the same but get the result $\cL_5 \precsim \frac{1}{\sqrt{m}}$.

Also, when the downstream satisfy the model (\ref{eq:targetdataregf}) : $Y_{T} = f^{*}_{T}(\hstar(X^{T})) + \varepsilon_{2}$. And by assuming $f^{*}_{T}\in \cH^{\beta}$, we can get the total excess risk of downstream task upper bound $\tlO \left(n^{-\frac{\beta}{2(d+1+\beta)}} \mathbb{I}_{\beta \leq 2} + n^{-\frac{\beta}{2d + 3\beta}} \mathbb{I}_{\beta > 2}\right) +  \tlO \left(m^{-\frac{\beta}{2(r+1+2\beta)}}\right)$.

\end{proof}

\begin{lemma}
\label{cL_4}
    \begin{equation}
    \begin{aligned}
        \cL_4= & \mE[\cL(\hat{\h}, Q, \F_{T}^{*}) - \cL(\hstar, Q, \F_{T}^{*})] \\
        \precsim & \tlO \left(n^{-\frac{\beta}{2(d+1+\beta)}} \mathbb{I}_{\beta \leq 2} + n^{-\frac{\beta}{2d + 3\beta}} \mathbb{I}_{\beta > 2}\right).
    \end{aligned}
    \end{equation}
\end{lemma}

\begin{proof}[Proof of \ref{cL_4}]

 We can utilize the property of Wasserstein distance by the distribution of $\hstar(\X) $ is exactly the same as $\cU_r$ and A(8) in the Assumption \ref{assump: downstream smooth} to get
\begin{align*}
            \mE[\cW_1(\hat{\h}(\XT),\hstar(\XT))] \leq & \mE[\cW_1(\hat{\h}(\XT), \hat{\h}(X)) + \cW_1( \hat{\h}(X),\hstar(\X)) +\cW_1( \hstar(\X), \hstar(\XT))] \\
            = & \mE[ \cW_1(\hat{\h}(\XT), \hat{\h}(X)) + \cW_1( \hat{\h}(X),\cU_r) +\cW_1( \hstar(\X), \hstar(\XT))] \\
            \precsim & \tlO \left(n^{-\frac{\beta}{2(d+1+\beta)}} \mathbb{I}_{\beta \leq 2} + n^{-\frac{\beta}{2d + 3\beta}} \mathbb{I}_{\beta > 2}\right)  +  \left(  n^{\frac{d+1}{2(d+1+\beta)}} \mathbb{I}_{\beta \leq 2} + n^{\frac{d+1}{2d + 3\beta}} \mathbb{I}_{\beta > 2} \right)\omega \\
            \precsim & \tlO \left(n^{-\frac{\beta}{2(d+1+\beta)}} \mathbb{I}_{\beta \leq 2} + n^{-\frac{\beta}{2d + 3\beta}} \mathbb{I}_{\beta > 2}\right).
\end{align*}
  The first inequality is by the distance property of Wasserstein distance. The first equality is by the optimality of the function $(f_0, \hstar)$. For the second inequality, we can get from lemma \ref{bounding the W1(hstar, hat h)}. Also, the coefficient before $\omega$ is due to the norm constrained of $\hat{\h}$.
 
Then we denote function $ g_{y_0}(z) = \ell(z, y_{0})$. By A(7) in Assumption \ref{assump: downstream smooth}, we can show that $g_{y_0}(z)$ is a $L_0$-lipschitz function.

\begin{align*}
    \cL_4 = & \mE_{\cD}[\cL(\hat{\h}, Q, \F_{T}^{*}) - \cL(\hstar, Q, \F_{T}^{*})] \\
          = & \mE_{\cD}[\mE_{Y \sim \YT}[\mE_{Z \sim (\F_{T}^{*}\hat{\h}(\XT) + q(A \XT)) | Y}[g_{Y}(Z)|Y] - \mE_{Z' \sim (\F_{T}^{*}\hstar(\XT) + q(A \XT))|Y}[g_{Y}(Z') | Y]]] \\
            \leq & \mE_{\cD}[\mE_{\YT}[L_0 \cW_1((\F_{T}^{*}\hat{\h}(\XT)+ q(A \XT))|\YT, (\F_{T}^{*}\hstar(\XT)+ q(A \XT))|\YT)]] \\
            = & L_0 \mE_{\cD}[\mE_{ \YT}[\cW_1((\F_{T}^{*}\hat{\h}(\XT)+ q(A \XT))|\YT, (\F_{T}^{*}\hstar(\XT)+ q(A \XT))|\YT)]]\\
            = & L_0 \mE_{\cD}[ \mE_{ \YT}[\cW_1(\F_{T}^{*}\hat{\h}(\XT)|\YT, \F_{T}^{*}\hstar(\XT)|\YT)]].
\end{align*}

The first inequality is due to the definition of Wasserstein distance.
Then by the A(9) in Assumption \ref{assump: downstream smooth} we can get

\begin{align*}
    \cL_4 \leq & L_0 \mE_{\cD}[ \mE_{ \YT}[\cW_1(\F_{T}^{*}\hat{\h}(\XT)|\YT, \F_{T}^{*}\hstar(\XT)|\YT)]] \\
     \leq &  L_0 \nu  \mE_{\cD}[\cW_1(\hat{\h}(\XT),\hstar(\XT)) ]\\
     \leq & \tlO \left(n^{-\frac{\beta}{2(d+1+\beta)}} \mathbb{I}_{\beta \leq 2} + n^{-\frac{\beta}{2d + 3\beta}} \mathbb{I}_{\beta > 2}\right).
\end{align*}
\end{proof}

\subsection{Proof of the Theorem \ref{theorem: special case of sparse erm during fine-tuning}}
The proof of Theorem \ref{theorem: special case of sparse erm during fine-tuning} is almost the same as proof in Section \ref{Proof of sparsity of erm upstream}.
\begin{proof}[Proof of the Theorem \ref{theorem: special case of sparse erm during fine-tuning} ]

Following almost the same procedure in the proof of Theorem \ref{Theorem: sparsity of erm solutions}. 
We consider the case of "complete transfer" where $Q^{*}=0$, no extra network and the loss $\ell(\cdot, \cdot)$ is the least square form such that $\hat{\cL}(\hat{\h}, \F) = \sum_{i=1}^{m}\ell(\F \hat{\h}(\x_{T,i}), Y_{T,i})/m + \chi \| \F \|_{1} = \sum_{i=1}^{m}(\F \hat{\h}(\X_{T,i})- Y_{T,i})^{2}/m + \chi \| \F \|_{1}$. 

In this case the ERM solution is a vector $\hat{\F}_{T} \in \arg\min_{\F \in \mathcal{F}} \hat{\cL}(\hat{\h}, \F)$. By the definition of $\hat{\F}_{T}$, we know that $\hat{\cL}(\hat{\F}_{T}, \hat{\h}) \leq \hat{\cL}(\F_{T}^{*}, \hat{\h})$. 
    \[ \sum_{i=1}^{m} \frac{1}{m} (Y_{T,i} - \hat{\F}_{T}\hat{\h}(\X_{T,i}))^2 + \chi \| \hat{\F}_{T}\|_{1} \leq \sum_{i=1}^{m} \frac{1}{m} (Y_{T,i} - F^{*}_{T}\hat{\h}(\X_{T, i}))^2 + \chi  \| \F^{*}_{T}\|_{1} .
    \]

\begin{equation}
    \begin{aligned}
    \label{primal sparsity down}
    \sum_{i=1}^{m} \frac{1}{m} (Y_{T,i} - \hat{\F}_{T}\hat{\h}(\X_{T,i}))^2  \leq \sum_{i=1}^{m} \frac{1}{m}  (Y_{T,i} - F^{*}_{T}\hat{\h}(\X_{T,i}))^2 + \chi  \| \hat{\F}_{T} - \F^{*}_{T}\|_{1}.
    \end{aligned}
\end{equation}

By the equation 
\begin{equation}
\begin{aligned}
\label{square term down}
 \sum_{i=1}^{m} \frac{1}{m} (Y_{T,i} - \hat{\F}_{T}\hat{\h}(\X_{T,i}))^2 & = \sum_{i=1}^{m} \frac{1}{m} (Y_{T,i} - \F^{*}_{T}\hat{\h}(\X_{T,i}) +  \F^{*}_{T}\hat{\h}(\X_{T,i})- \hat{\F}_{T}\hat{\h}(\X_{T,i}))^2 \\
 & = \sum_{i=1}^{m} \frac{1}{m} (Y_{T,i} - \F^{*}_{T}\hat{\h}(\X_{T,i}))^2 \\
 & \quad + \sum_{i=1}^{m} \frac{2}{m}(Y_{T,i} - \F^{*}_{T}\hat{\h}(\X_{T,i}))(\F^{*}_{T}\hat{\h}(\X_{T,i})- \hat{\F}_{T}\hat{\h}(\X_{T,i})) \\
 & \quad +  \sum_{i=1}^{m} \frac{1}{m}((\F^{*}_{T} -\hat{\F}_{T})\hat{\h}(\X_{T,i}))^{2}, \\
\end{aligned} 
\end{equation}
after estimating the last two term in right part of upper equation. We can replace the left part in \eqref{primal sparsity down} by \eqref{square term down} to get the desired result by selecting proper $\chi$.

Then by lemma \ref{lemma: bounding hat h and hstar at downstream}, we know that 
\begin{align*}
    &\quad \mE_{\cD}[\sum_{i=1}^{m} \frac{2}{m}(Y_{T, i} - \F^{*}_{T}\hat{\h}(\X_{T,i}))(\F^{*}_{T}\hat{\h}(\X_{T,i})- \hat{\F}_{T}\hat{\h}(\X_{T,i}))]  \\
    &=  \mE_{\cD}[\sum_{i=1}^{m} \frac{2}{m}( \F^{*}_{T}\hstar(\X_{T, i}) - \F^{*}_{T}\hat{\h}(\X_{T, i}) + \epsilon_{i})(\F^{*}_{T}- \hat{\F}_{T})\hat{\h}(\X_{T, i}) ]\\
    &=   \mE_{\cD}[ \sum_{i=1}^{m} \frac{2}{m}(\F^{*}_{T}(\hstar(\X_{T, i}) -\hat{\h}(\X_{T, i})) + \epsilon_{i})(\F^{*}_{T}- \hat{\F}_{T})\hat{\h}(\X_{T, i}) ]\\
    &\leq  \mE_{\cD}[\|\F^{*}_{T}- \hat{\F}_{T}\|_{1} 
    \|\sum_{i=1}^{m}\frac{2}{m} \F^{*}_{T}(\hstar(\X_{T, i}) -\hat{\h}(\X_{T, i}))\hat{\h}(\X_{T, i})\|_{\infty} \\
    & \quad + \|\F^{*}_{T}- \hat{\F}_{T}\|_{1} 
    \|\sum_{i=1}^{m}\frac{2}{m} \epsilon_{i} \hat{\h}(\X_{T, i})\|_{\infty} ]\\
    & \overset{w.h.p}{\leq}  \|\F^{*}_{T}- \hat{\F}_{T}\|_{1} \left( B(\sqrt{\frac{2 \log \frac{2}{\delta}}{m}} + B_{2}  \gamma_{1} n^{-t(\beta)} )+ B_{2}  \gamma_{1} n^{-t(\beta)} + \sqrt{\frac{2 \log \frac{2}{\delta}}{m}} +  \sqrt{\frac{2 \log \frac{2r}{\delta}}{m}} \right),\\
\end{align*}
where $\epsilon_{i}$ here is the \textit{i.i.d} sample of $\varepsilon_{2}$ and for simplicity we omit the footnote. 

By A(10) in Assumption \ref{Compatibility condition for downstream}, we can get 
\[
\mE_{\cD}[\frac{1}{m}\sum_{i=1}^{m} ((\F_{T}^{*} -\hat{\F}_{T})\hat{\h}(\X_{T, i}))^{2}]  = \mE_{\cD} [\|(\F^{*}_{T} -\hat{\F}_{T}) \hat{\H}_{T}\|_{2}^{2} ]\geq B_{1} \|\F^{*}_{T} -\hat{\F}_{i}\|_{T}^{2}.
\]

So the \eqref{primal sparsity down} can be further deduced as followed:
\begin{align*}
     B_{1}\|&\F^{*}_{T} -\hat{\F}_{T}\|_{2}^{2} 
      \overset{w.h.p}{\leq}  \chi   \| \hat{\F}_{T} - \F^{*}_{T}\|_{1} \\
     & + \|\F^{*}_{T}- \hat{\F}_{T}\|_{1} \left(B(\sqrt{\frac{2 \log \frac{2}{\delta}}{m}} + B_{2}  \gamma_{1} n^{-t(\beta)} )+ B_{2}  \gamma_{1} n^{-t(\beta)} + \sqrt{\frac{2 \log \frac{2}{\delta}}{m}} +  \sqrt{\frac{2 \log \frac{2r}{\delta}}{m}}\right)  \\
 \end{align*}
\begin{align*}
          B_{1}&\|\F^{*}_{T} -\hat{\F}_{T}\|_{2}^{2} \\
          \overset{w.h.p}{\leq}& \Big(B(\sqrt{\frac{2 \log \frac{2}{\delta}}{m}} + B_{2}  \gamma_{1} n^{-t(\beta)} )+ B_{2}  \gamma_{1} n^{-t(\beta)} + \sqrt{\frac{2 \log \frac{2}{\delta}}{m}} +  \sqrt{\frac{2 \log \frac{2r}{\delta}}{m}} + \chi \Big)\| \hat{\F}_{T} - \F^{*}_{T}\|_{1}  \\
         \overset{w.h.p}{\leq}  &  \sqrt{r} \Big(B(\sqrt{\frac{2 \log \frac{2}{\delta}}{m}} + B_{2}  \gamma_{1} n^{-t(\beta)} )+ B_{2}  \gamma_{1} n^{-t(\beta)} + \sqrt{\frac{2 \log \frac{2}{\delta}}{m}} +  \sqrt{\frac{2 \log \frac{2r}{\delta}}{m}} +\chi \Big) \| \hat{\F}_{T} - \F^{*}_{T}\|_{2}  \\
\end{align*}
Thus, we have 
\begin{align*}
           B_{1}\|\F^{*}_{T} -&\hat{\F}_{T}\|_{2}\\ \overset{w.h.p}{\leq}  &  \sqrt{r} \Big(  B(\sqrt{\frac{2 \log \frac{2}{\delta}}{m}} + B_{2}  \gamma_{1} n^{-t(\beta)} )+ B_{2}  \gamma_{1} n^{-t(\beta)} + \sqrt{\frac{2 \log \frac{2}{\delta}}{m}} +  \sqrt{\frac{2 \log \frac{2r}{\delta}}{m}} + \chi \Big) \\
\end{align*}
By set $\chi = \frac{1}{\sqrt{m}}$, we can get the wanted result which means w.p.t that the sparsity of $\hat{\F}_{T}$ is similar with $\F_{T}^{*}$

\end{proof}

\begin{lemma}
        For any $\delta > 0$,  $i \in [p]$, with probability at least $1- 2\delta$ we get:
    \begin{align*}
     \mE_{\cD}[ \|\sum_{i=1}^{m}\frac{1}{m} \F^{*}_{T}(\hstar(x_{T,i}) -\hat{\h}(x_{T,i}))\hat{\h}(x_{T,i})\|_{\infty}] \leq &  B(\sqrt{\frac{2 \log \frac{2}{\delta}}{m}} + B_{2}  \gamma_{1} n^{-t(\beta)} ) , \\
     \mE_{\cD}[ \|\sum_{i=1}^{m}\frac{1}{m} \epsilon_{i} \hat{\h}(x_{T,i})\|_{\infty} ]\leq &  B_{2}  \gamma_{1} n^{-t(\beta)} + \sqrt{\frac{2 \log \frac{2}{\delta}}{m}} +  \sqrt{\frac{2 \log \frac{2r}{\delta}}{m}},
\end{align*} where $\gamma_{1}$ is a constant big enough.
    \label{lemma: bounding hat h and hstar at downstream}
\end{lemma}

\begin{proof}[Proof of Lemma \ref{lemma: bounding hat h and hstar at downstream}]
By A(6) in Assumption \ref{assump: downstream smooth} we know that,
    \begin{align*}
    \|\sum_{i=1}^{m}\frac{1}{m} \F^{*}_{i}(\hstar(\X_{T, i}) -\hat{\h}(\X_{T, i}))\hat{\h}(\X_{T, i})\|_{\infty} \leq & B 
    \sum_{i=1}^{m}\frac{1}{m}\|\hstar(\X_{T, i}) -\hat{\h}(\X_{T, i})\|_{1},
    \end{align*}
    \begin{align*}
    \|\sum_{i=1}^{m}\frac{1}{m} \epsilon_{i} \hat{\h}(\X_{T, i})\|_{\infty} \leq &
    \sum_{i=1}^{m}\frac{1}{m} \| \epsilon_{i}\|_{\infty} \|\hstar(\X_{T, i}) -\hat{\h}(\X_{T, i})\|_{1}  +   \|\sum_{i=1}^{m}\frac{1}{m} \epsilon_{i} \hstar(\X_{T, i})\|_{\infty},
    \end{align*}

    By lemma \ref{cL_4}, we know that $ \mE[\cW_1(\hat{\h}(\XT),\hstar(\XT))] \leq \tlO \left(n^{-\frac{\beta}{2(d+1+\beta)}} \mathbb{I}_{\beta \leq 2} + n^{-\frac{\beta}{2d + 3\beta}} \mathbb{I}_{\beta > 2}\right)$. For the simplicity of calculation, we can set a constant $\gamma_{1} >$ 0 big enough that $\mE[\cW_1(\hat{\h}(\XT),\hstar(\XT))] \leq \gamma_{1}  n^{-t(\beta)}$.

    By A(11) in the Assumption \ref{Compatibility condition for downstream}, $\mE[ \|\hstar(\XT^{'}) -\hat{\h}(\XT^{'})\|_{1}] \leq B_{2}  n^{-t(\beta)} $.
     By the independence of dataset $\cD$ and $\XT^{'}$. We can use the Hoeffding’s inequality to get:
        \begin{align*}
        \Pr\left[ \left| \sum_{i=1}^{m}\frac{1}{m}\|\hstar(\X_{T, i}) -\hat{\h}(\X_{T, i})\|_{1} - \mE_{\XT^{'}}[ \|\hstar(\XT^{'}) -\hat{\h}(\XT^{'})\|_{1}] \right| \Big| \cD \right] \geq \delta_{1}] \leq  2\exp\left(-\frac{\delta_{1}^{2} m}{2}\right).
    \end{align*}
    Then set $\delta_1 = \sqrt{\frac{2 \log \frac{2}{\delta}}{m}}$, for any draw of dataset $\cD$ with probability at least $1- \delta$, we know $ \sum_{i=1}^{m}\frac{1}{m}\|\hstar(\X_{T, i}) -\hat{\h}(\X_{T, i})\|_{1} - \mE_{\XT^{'}}[ \|\hstar(\XT^{'}) -\hat{\h}(\XT^{'})\|_{1}] ]  \leq  \delta_{1}$. So
    \begin{align*}
           \mE_{\cD}[ \sum_{i=1}^{m}\frac{1}{m}\|\hstar(\X_{T, i}) -\hat{\h}(\X_{T, i})\|_{1} - \mE_{\XT^{'}}[ \|\hstar(\XT^{'}) -\hat{\h}(\XT^{'})\|_{1}] ]  \leq & \delta_{1}, \\
            \mE_{\cD}[ \sum_{i=1}^{m}\frac{1}{m}\|\hstar(\X_{T, i}) -\hat{\h}(\X_{T, i})\|_{1} - \mE_{\XT^{'}}[ \|\hstar(\XT^{'}) -\hat{\h}(\XT^{'})\|_{1}]  ] \leq & \sqrt{\frac{2 \log \frac{2}{\delta}}{m}}. \\
    \end{align*}
    Move the term $\mE_{\XT^{'}}[ \|\hstar(\XT^{'}) -\hat{\h}(\XT^{'})\|_{1}]$ to the right side of the equation and use the assumption (A5).
    \begin{align*}
        \mE_{\cD}[ \sum_{i=1}^{m}\frac{1}{m}\|\hstar(\X_{T, i}) -\hat{\h}(\X_{T, i})\|_{1}] \leq & \sqrt{\frac{2 \log \frac{2}{\delta}}{m}} + \mE[ \|\hstar(\XT^{'}) -\hat{\h}(\XT^{'})\|_{1}],\\
        \mE_{\cD}[ \sum_{i=1}^{m}\frac{1}{m}\|\hstar(\X_{T, i}) -\hat{\h}(\X_{T, i})\|_{1}] \leq & \sqrt{\frac{2 \log \frac{2}{\delta}}{m}} + B_{2} \mE[ \cW_{1}(\hstar(\XT^{'}), \hat{\h}(\XT^{'}))], \\
        \mE_{\cD}[ \sum_{i=1}^{m}\frac{1}{m}\|\hstar(\X_{T, i}) -\hat{\h}(\X_{T, i})\|_{1}] \leq & \sqrt{\frac{2 \log \frac{2}{\delta}}{m}} + B_{2}  \gamma_{1} n^{-t(\beta)}.\\
    \end{align*}
Also,
    \begin{align*}
        &  \Pr[\| \sum_{j=1}^{m}\frac{1}{m} \epsilon_{i}\hstar(\X_{T, i}) \|_{\infty}\geq \delta_{2}]\\
        =& \Pr[ \| \sum_{j=1}^{m}\frac{1}{m} \epsilon_{i}\hstar(\X_{T, i}) - \mE[ \varepsilon_{1}\hstar(\XT)]\|_{\infty} \geq \delta_{2}]\\
        =&   \Pr[\exists k \in [r], | \sum_{j=1}^{m}\frac{1}{m} \epsilon_{i}h^{*}_{k}(\X_{T, i}) | \geq \delta_{2}] \\
        =&   \Pr[\exists k \in [r], | \sum_{j=1}^{m}\frac{1}{m} \epsilon_{i}h^{*}_{k}(\X_{T, i}) - \mE[\varepsilon_{1} h^{*}_{k}(X)]| \geq \delta_{2}] \\
        \leq & 2 r \exp(-\frac{\delta_{2}^{2} m}{2}).\\
    \end{align*}
    Then we can set the $\delta_{2} = \sqrt{\frac{2 \log \frac{2r}{\delta}}{m}}$ here, then with probability at least $1-\delta$, we know $\| \sum_{j=1}^{m}\frac{1}{m} \epsilon_{i}\hstar(\X_{T, i}) \|_{\infty} \leq \sqrt{\frac{2 \log \frac{2r}{\delta}}{m}} $.

The again by A(11) in the Assumption \ref{Compatibility condition for downstream}, with probability at least $1- 2\delta$ we get:
    \begin{align*}
     &\mE_{\cD} [\|\sum_{i=1}^{m}\frac{1}{m} \F^{*}_{i}(\hstar(\X_{T, i}) -\hat{\h}(\X_{T, i}))\hat{\h}(\X_{T, i})\|_{\infty}] \\
    \leq & B \mE_{\cD}[\sum_{i=1}^{m}\frac{1}{m}\|\hstar(\X_{T, i}) - \hat{\h}(\X_{T, i})\|_{1}] \leq B(\sqrt{\frac{2 \log \frac{2}{\delta}}{m}} + B_{2}  \gamma_{1} n^{-t(\beta)} ) .\\
    \end{align*}
and following the same way before
\begin{align*}
    \mE_{\cD}[\|\sum_{i=1}^{m}\frac{1}{m} \epsilon_{i} \hat{\h}(\X_{T, i})\|_{\infty}] 
     \leq & \mE_{\cD} [\sum_{i=1}^{m}\frac{1}{m} \| \epsilon_{i}\|_{\infty} \|\hstar(\X_{T, i}) -\hat{\h}(\X_{T, i})\|_{1}  +   \|\sum_{i=1}^{m}\frac{1}{m} \epsilon_{i} \hstar(\X_{T, i})\|_{\infty}]\\
     \leq &  B_{2}  \gamma_{1} n^{-t(\beta)} + \sqrt{\frac{2 \log \frac{2}{\delta}}{m}} +  \sqrt{\frac{2 \log \frac{2r}{\delta}}{m}}.
\end{align*}

\end{proof}

\section{Theorem cited during our proof}
\label{Theorem cited when proof}

\begin{theorem}[Theorem 3.2 in \cite{jiao2023approximation}]
\label{thm: appro for norm-constrain NN}
Let $d \in \mathbb{N}$ and $\alpha=s+\beta>0$, where $s \in \mathbb{N}_0$ and $\beta \in(0,1]$.

(1) There exists $c>0$ such that for any $K \geq 1$, any $W \geq c K^{(2 d+\alpha) /(2 d+2)}$ and $L \geq$ $2\left\lceil\log _2(d+s)\right\rceil+2$,
\[
\mathcal{E}\left(\mathcal{H}^\alpha, \mathcal{N N}(W, L, K)\right) \lesssim K^{-\alpha /(d+1)}.
\]
(2) If $d>2 \alpha$, then for any $W, L \in \mathbb{N}, W \geq 2$ and $K \geq 1$,
\[
\mathcal{E}\left(\mathcal{H}^\alpha, \mathcal{N} \mathcal{N}(W, L, K)\right) \gtrsim(K \sqrt{L})^{-2 \alpha /(d-2 \alpha)}.
\]
\end{theorem}

By $\mathcal{S N N}_{d, 1}(W, L, K) \subseteq \mathcal{ N N}_{d, 1}(W, L, K) \subseteq \mathcal{S N N}_{d, 1}(W + 1, L, K)$ (Proposition 2.1 in \cite{jiao2023approximation}),  the subsequent theorem can be employed during the proof of Theorem \ref{main resultup} and Theorem \ref{fine-tuning result}.
\begin{theorem}[Lemma 2.3 in \cite{jiao2023approximation}]
\label{thm: stata for norm-constrained nn}
For any $\boldsymbol{x}_1, \ldots, \boldsymbol{x}_n \in[-B, B]^d$ with $B \geq 1$, let $S:=\left\{\left(\phi\left(\boldsymbol{x}_1\right), \ldots, \phi\left(\boldsymbol{x}_n\right)\right): \phi \in \mathcal{S N N}_{d, 1}(W, L, K)\right\} \subseteq \mathbb{R}^n$, then
\[
\mathcal{R}_n(S) \leq \frac{1}{n} K \sqrt{2(L+2+\log (d+1))} \max _{1 \leq j \leq d+1} \sqrt{\sum_{i=1}^n x_{i, j}^2} \leq \frac{B K \sqrt{2(L+2+\log (d+1))}}{\sqrt{n}},
\]
where $x_{i, j}$ is the $j$-th coordinate of the vector $\tilde{\boldsymbol{x}}_i=\left(\boldsymbol{x}_i^{\top}, 1\right)^{\top} \in \mathbb{R}^{d+1}$. When $W \geq 2$,
\[
\mathcal{R}_n(S) \geq \frac{K}{2 \sqrt{2} n} \max _{1 \leq j \leq d+1} \sqrt{\sum_{i=1}^n x_{i, j}^2} \geq \frac{K}{2 \sqrt{2 n}}.
\]
\end{theorem}

\section{Experimental details.}
\label{sec: Experimental details}
In this section, we briefly highlight some important settings in our experiments. More details can be found in our released code in supplementary material. 


\textbf{Hyperparameter search.}
For the classification task, we maintained a fixed batch size of 64 and utilized the Adam \citep{kingma2014adam} optimizer with the default setting with the only different in learning rate. Both upstream and downstream training consisted of 50 epochs. The upstream training employed a constant learning rate of 0.00001 with the Adam optimizer. For the downstream task, the learning rate for the fully connected layers was set at 0.1, and using the CosineAnnealingLR scheduler. Also, the learning rate for EfficientNet’s feature extractor, which is the Q network, is 0.00001. The hyperparameter search focused solely on the penalty parameters, with values for $\lambda$ selected from \{5, 10, 15\}, $\tau$ from \{0, 10\} and $\mu$ from \{0, 10\}. The hyperparameter for the independence penalty term in the downstream task, $\kappa$ is from \{5, 10, 15\}. 

For the regression task, we maintained a fixed batch size of 50 and employed the Adam optimizer with settings identical to those used in the classification task, except for a variation in the learning rate. The upstream training consisted of 50 epochs and downstream training consisted of 100 epochs. The upstream training employed a learning rate of 0.2. For the downstream task, the learning rate for the fully connected layers was set at 0.005. Both training processes using the CosineAnnealingLR scheduler. The hyperparameter search focused solely on the penalty parameters, with values for $\lambda$ selected from \{0.5, 1.0, 1.5, 2.0\}, $\tau$ from \{0, 0.2\} and $\mu$ from \{0, 0.5\}. 
The parameter \(\kappa\) for the independence penalty term in the downstream task is selected from \(\{0.0625, 0.125, 0.25\}\), and \(\chi\) for the $\ell_{1}$ penalty term is selected from \(\{0.03125, 0.0625\}\).
 To prevent certain components of the penalty terms from disproportionately affecting gradient calculations due to their larger values compared to other penalty terms, and to facilitate parameter selection, the penalty terms are normalized by their detached versions. The total loss is then computed as a weighted sum of the individual penalty terms, as follows:
$
\text{penalty}_1 / \text{penalty}_1.detach() + \lambda * \text{penalty}_2 / \text{penalty}_2.detach()  + \tau * \text{penalty}_2 / \text{penalty}_2.detach() + \mu *\text{penalty}_3/ \text{penalty}_3.detach()
$.
Due to constraints in computational resources and time, we did not explicitly explore the impact of $\zeta$ and $d^{*}$ on downstream task performance. 

\textbf{Details about comparison with other works.} 
Due to time constraints and issues with the code logic, we aligned all training details as closely as possible except for the data input method. In the DomainBed codebase \citep{gulrajani2020search}, the input data is structured such that each remaining domain inputs one batch size of samples, while our input method shuffles all remaining domains and inputs one batch size of samples. Therefore, to align as closely as possible, we set the batch size for upstream training to 21 and the batch size for downstream training to 64 when conducting comparative experiments in the DomainBed codebase. The typical hyperparameters for the compared algorithms are set to their default values in the DomainBed codebase.

To make a comparison with the method put forward in \cite{cai2021transfer}, we chose solely the "dog" and "elephant" photos within the PACS dataset to carry out our experiments. 
Regarding the KNN-based method, we split the downstream domain in an 80\%-20\% proportion to form a training set and a test set respectively. Subsequently, we utilized the data from all the remaining domains as well as the downstream training data to construct a KNN classifier for the test samples. However, the high input dimension of photo data and relatively less data makes it impossible for the break rule in Algorithm 2 of \cite{cai2021transfer} to be satisfied. To simplify the computation process, we choose to select the  values for the four domains from among \{3, 3, 3, 3\}, \{4, 4, 4, 4\}, $\ldots$, \{9, 9, 9, 9\}, and \{10, 10, 10, 10\}. We only display the highest accuracy figures in Table \ref{The accuracy comparad with KNN}. 

\textbf{Details of experiments exploring the effect of representation dimensionality}
To highlight the impact of sample size on estimation error, we trained the downstream models on OfficeHome using only 15\% of the original training data. To investigate the influence of learning functions of different dimensions for downstream task model \eqref{eq:targetdataregf}, we added three extra layers to the downstream network. Similarly, for the Wear dataset, we added two extra layers. For both datasets, we used cross-validation to select the optimal $\ell_{1}$ regularization hyperparameter. 

\textbf{Error bar.}
Due to limited computational resources, we conducted downstream training three times using the checkpoints from the initial upstream training. For comparative results with other methods, particularly those using the DomainBed code \citep{gulrajani2020search}, we repeated the tests by only varying the seed. The error bar represents the standard deviation of the best result from these experimental runs.

\textbf{Experiments Compute Resources.}
Our experiments were conducted on Nvidia DGX Station workstation using a single Tesla V100 GPU unit. For a specific domain in the image dataset, the experimental time for a specific parameter ranges from 0.5 hours to 1.5 hours. Therefore, estimating the experimental results with error bars would require at least three days.

For additional details regarding the small hand-designed model architecture, please refer to the code provided in the supplementary materials.

\end{document}